\documentclass[twoside,11pt]{article}

\usepackage{blindtext}

\usepackage{jmlr2e}
\usepackage{tikz}
\usepackage{amsmath}
\usepackage{amssymb}
\usepackage{mathrsfs}
\usepackage{booktabs,cleveref}
\usetikzlibrary{arrows.meta, positioning, patterns, calc, backgrounds, shapes.geometric, shadows, decorations.pathmorphing, fit, 3d}
\usetikzlibrary{decorations.pathreplacing}
\usepackage{algpseudocode,algorithm}

\definecolor{trapblue}{RGB}{52, 101, 164}
\definecolor{trapbg}{RGB}{232, 240, 254}
\definecolor{funnelred}{RGB}{192, 57, 43}
\definecolor{funnelbg}{RGB}{253, 237, 236}
\definecolor{wallgray}{RGB}{100, 110, 125}
\definecolor{agentblue}{RGB}{41, 128, 185}
\definecolor{arrowgray}{RGB}{130, 140, 155}
\definecolor{labelgray}{RGB}{80, 90, 105}

\newtheorem{assumption}[theorem]{Assumption}
\newcommand{\UM}{\textsf{UM}}

\newcommand{\Lib}{\mathscr{L}}
\newcommand{\Width}{\mathcal{W}}
\newcommand{\R}{\mathbb{R}}

\newcommand{\eps}{\varepsilon}

\usepackage{lastpage}
\jmlrheading{23}{2026}{1-\pageref{LastPage}}{1/21; Revised 5/22}{9/22}{21-0000}{Li}

\ShortHeadings{Structural Learning Theory}{Li}
\firstpageno{1}

\begin{document}

\title{Structural Learning Theory: A Metric-Topology Factorization Approach}

\author{\name Xin Li \email xli48@albany.edu \\
       \addr Department of Computer Science\\
       University at Albany\\
       Albany, NY 12222, USA
       }

\editor{TBD}

\maketitle

\newcommand{\figTrapFunnel}{%
\begin{figure}[h]
  \centering
  \resizebox{.66\linewidth}{!}{
\begin{tikzpicture}[
  >=Stealth,
  every node/.style={font=\small},
  thick
]
 
\begin{scope}[shift={(-4.6,0)}]
  \fill[trapbg, rounded corners=6pt] (-0.25,-0.22) rectangle (3.45,3.45);
  \draw[wallgray, line width=1.6pt, rounded corners=2pt]
    (0,0) rectangle (3.2,3.2);
  \draw[wallgray, line width=1.3pt] (0,2.2)   -- (1.8,2.2);
  \draw[wallgray, line width=1.3pt] (1.2,0)   -- (1.2,1.4);
  \draw[wallgray, line width=1.3pt] (1.8,0.9) -- (3.2,0.9);
  \draw[wallgray, line width=1.3pt] (1.8,1.4) -- (1.8,2.2);
  \draw[wallgray, line width=1.3pt] (0.5,0.9) -- (1.2,0.9);
  \draw[wallgray, line width=1.3pt] (2.3,1.4) -- (3.2,1.4);
  \draw[wallgray, line width=1.3pt] (0,1.4)   -- (0.6,1.4);
  \foreach \x/\y/\lab in {
    0.55/2.7/$B_1$,
    2.5/2.7/$B_2$,
    0.55/1.1/$B_3$,
    2.5/0.45/$B_4$}
  { \node[funnelred, font=\scriptsize\bfseries] at (\x,\y) {\lab}; }
  \fill[agentblue] (0.55,0.45) circle (0.14);
  \draw[agentblue!40!white, line width=0.5pt] (0.55,0.45) circle (0.22);
  \node[trapblue!80!black, font=\large\bfseries] at (1.62,1.18) {?};
  \node[trapblue!80!black, font=\normalsize\bfseries] at (1.6,-0.55) {T\,R\,A\,P};
  \node[labelgray, font=\scriptsize, text width=4.1cm, align=center]
    at (1.6,-1.15) {Which basin is in?\\[-1pt](discrete/combinatorial)};
\end{scope}
 
\draw[->, arrowgray, line width=1.8pt] (-0.95,1.6) -- (0.55,1.6);
\node[labelgray, font=\scriptsize, align=center] at (-0.20,2.05) {identify};
\node[labelgray, font=\scriptsize, align=center] at (-0.20,1.18) {context};
 
\begin{scope}[shift={(0.9,0)}]
  \fill[funnelbg, rounded corners=6pt] (-0.25,-0.22) rectangle (3.45,3.45);
  \draw[wallgray, line width=1.6pt]
    (0,3.1) .. controls (0,1.0) and (1.1,0.05) .. (1.6,0.05)
             .. controls (2.1,0.05) and (3.2,1.0) .. (3.2,3.1);
  \fill[funnelred!6]
    (0,3.1) .. controls (0,1.0) and (1.1,0.05) .. (1.6,0.05)
             .. controls (2.1,0.05) and (3.2,1.0) .. (3.2,3.1) -- cycle;
  \draw[->, agentblue!75, line width=1.1pt] (0.45,2.85) -- (1.00,2.05);
  \draw[->, agentblue!75, line width=1.1pt] (2.75,2.85) -- (2.20,2.05);
  \draw[->, agentblue!75, line width=1.1pt] (0.75,2.00) -- (1.15,1.25);
  \draw[->, agentblue!75, line width=1.1pt] (2.45,2.00) -- (2.05,1.25);
  \draw[->, agentblue!75, line width=1.1pt] (1.10,1.20) -- (1.40,0.55);
  \draw[->, agentblue!75, line width=1.1pt] (2.10,1.20) -- (1.80,0.55);
  \fill[funnelred] (1.6,0.18) circle (0.12);
  \draw[funnelred!50, line width=0.5pt] (1.6,0.18) circle (0.21);
  \node[funnelred, font=\scriptsize\bfseries, above=4pt] at (1.6,0.18) {$G_c^\star$};
  \node[funnelred!80!black, font=\normalsize\bfseries] at (1.6,-0.55) {F\,U\,N\,N\,E\,L};
  \node[labelgray, font=\scriptsize, text width=3.1cm, align=center]
    at (1.6,-1.15) {Optimize within basin\\[-1pt](continuous/metric)};
\end{scope}
 
\end{tikzpicture}%
}
\caption{The trap-funnel decomposition.
    \textbf{Left:} The trap is the problem of determining which structural
    basin ($B_1,\dots,B_4$) governs the current input.
    This is combinatorial and has no useful gradient.
    \textbf{Right:} The funnel is the problem of converging to the local
    attractor $G_c^\star$ within the correct basin.
    This is continuous and amenable to gradient-based optimisation.}
  \label{fig:trap-funnel}
\end{figure}
}

\newcommand{\figBouquet}{%
\begin{figure}[h]
\centering
\begin{tikzpicture}[scale=1.0, every node/.style={font=\small}]
  \fill[black] (0,0) circle (0.08);
  \node[below, font=\scriptsize] at (0,-0.15) {$x_0$};
  
  \def\radius{1.1}
  
  \draw[thick, red!70] (0,0) ++ (90:\radius) circle (\radius);
  \node[red!70, font=\scriptsize\bfseries] at (0,2*\radius+0.15) {$C_1$};
  
  \draw[thick, blue!70] (0,0) ++ (30:\radius) circle (\radius);
  \node[blue!70, font=\scriptsize\bfseries] at ({2*\radius*cos(30)+0.2},{2*\radius*sin(30)+0.1}) {$C_2$};
  
  \draw[thick, green!60!black] (0,0) ++ (-30:\radius) circle (\radius);
  \node[green!60!black, font=\scriptsize\bfseries] at ({2*\radius*cos(-30)+0.2},{2*\radius*sin(-30)-0.1}) {$C_3$};
  
  \draw[thick, orange!80] (0,0) ++ (-90:\radius) circle (\radius);
  \node[orange!80, font=\scriptsize\bfseries] at (0,-2*\radius-0.2) {$C_4$};
  
  \draw[thick, purple!70] (0,0) ++ (210:\radius) circle (\radius);
  \node[purple!70, font=\scriptsize\bfseries] at ({2*\radius*cos(210)-0.2},{2*\radius*sin(210)-0.1}) {$C_5$};
  
  \draw[thick, teal!70] (0,0) ++ (150:\radius) circle (\radius);
  \node[teal!70, font=\scriptsize\bfseries] at ({2*\radius*cos(150)-0.2},{2*\radius*sin(150)+0.1}) {$C_6$};
  
  \node[font=\small, text width=5cm, align=center] at (6,0) {%
    $w = 6 = \beta_1(M)$\\[4pt]
    Connected space\\
    (Laplacian sees 1)\\[4pt]
    Width = 6\\
    (CS operator sees 6)
  };
  
  \draw[<-, gray, thick] (0,0) -- (3.5,0.8);
  \node[font=\scriptsize, gray] at (3.3,0.9) {shared basepoint};
\end{tikzpicture}
\caption{A bouquet of \(w=6\) circles. All circles share the common basepoint \(x_0\), so the space is connected and the graph Laplacian detects only one component. However, each circle requires its own contractive predictor, so the width is \(w=6=\beta_1(M)\). This is the canonical example for Theorem~\ref{thm:strict-hierarchy}.}
\label{fig:bouquet}
\end{figure}
}

\newcommand{\figPhaseTransition}{%
\begin{figure}[h]
\centering
\includegraphics[width=\linewidth]{phase_transition_SLT.png}
\caption{Phase transition in the star-spoke experiment ($\mathbb{R}^{12}$, $w=8$, ridge regression, 50 trials per configuration). \textbf{(a)}~MSE vs.\ structural budget $K$ for sample sizes $n \in \{100, 200, 400, 800, 1600, 3200, 6400\}$. A sharp drop occurs at $K = w = 8$ (dashed red line): for $K < w$, MSE remains high regardless of $n$; for $K \ge w$, MSE drops and follows standard $O(1/n)$ decay. \textbf{(b)}~MSE vs.\ sample size for selected values of $K$. Curves with $K < w$ (dashed) plateau---more data does not help because the error is structural. Curves with $K \ge w$ (solid) decrease steadily. \textbf{(c)}~Gap ratio $\mathrm{MSE}(K{=}7)/\mathrm{MSE}(K{=}8)$ vs.\ sample size. The ratio \emph{widens} with more data, from ${\sim}1\times$ at $n=100$ to ${\sim}26\times$ at $n=6400$. This widening is the opposite of statistical overfitting and is the signature of structural under-resolution predicted by Theorem~\ref{thm:phase-transition}: the metric learner improves within each cell as $n$ grows, making the structural limitation more visible, not less.}
\label{fig:phase-transition}
\end{figure}
}

\newcommand{\figSlingshot}{%
\begin{figure}[t]
\centering
\begin{tikzpicture}[
  scale=0.9,
  every node/.style={font=\small},
  block/.style={draw, thick, rounded corners, minimum height=1cm, minimum width=2.2cm, align=center},
  arrow/.style={->, thick, >=stealth}
]
  \node[block, fill=blue!10] (X) at (0, 0) {Input\\$X$};
  
  \node[block, fill=green!10] (phi) at (3.5, 0) {Embedding\\$\phi: X \to Z$};
  
  \node[block, fill=yellow!15] (Z) at (7, 0) {Latent space\\$(Z, d_Z)$\\$d_Z = 2$};
  
  \node[block, fill=red!10] (Sigma) at (3.5, 2.8) {Indexer\\$\Sigma: X \to \Delta(\mathcal{C})$};
  
  \node[block, fill=purple!10] (G0) at (7, 2.8) {Contraction maps\\$\{G_c^0\}$ (frozen)};
  
  \node[block, fill=orange!10] (pi) at (10.5, 0) {Readout\\$\pi: Z \to Y$};
  
  \node[block, fill=gray!10] (Y) at (10.5, 2.8) {Output\\$\hat{y}$};
  
  \draw[arrow] (X) -- (phi);
  \draw[arrow] (phi) -- (Z);
  \draw[arrow] (Z) -- (pi);
  \draw[arrow] (X) |- (Sigma);
  \draw[arrow] (Sigma) -- (G0);
  \draw[arrow] (G0) -- (Y);
  \draw[arrow, dashed, gray] (pi) -- (Y);
  \draw[arrow, dashed, gray] (Z) -- (G0) node[midway, right, font=\scriptsize] {CS op.};
  
  \node[font=\scriptsize, blue!70, text width=2cm, align=center] at (3.5, 1.2) {trained by\\spatial prediction};
  \node[font=\scriptsize, red!70, text width=2cm, align=center] at (3.5, 4.2) {trained by\\mismatch signal};
  \node[font=\scriptsize, purple!70] at (7, 4.2) {frozen};
  
  \draw[decorate, decoration={brace, amplitude=5pt, mirror}, thick, gray!50] 
    (-0.3, -0.7) -- (4.8, -0.7) node[midway, below=8pt, font=\scriptsize, gray] {$\nabla_{\theta_\Sigma} L_{\mathrm{funnel}} = 0$};
\end{tikzpicture}
\caption{The metric slingshot architecture. The embedding \(\phi\) maps inputs from the high-dimensional space \(X\) into the low-dimensional latent space \(Z\). The topological indexer \(\Sigma\) routes inputs to contexts using mismatch signals (not task loss). The contraction maps \(\{G_c^0\}\) are pre-built and frozen. Structural decoupling holds by construction: no task-loss gradient flows into the indexer.}
\label{fig:slingshot}
\end{figure}
}

\newcommand{\figLayerStructure}{%
\begin{figure}[t]
\centering
\begin{tikzpicture}[
  scale=0.85,
  every node/.style={font=\small},
  layer/.style={draw, thick, rounded corners=3pt, minimum height=1.1cm, align=center}
]
  \node[layer, fill=blue!12, minimum width=12cm] (L0) at (0, 0) {%
    \textbf{Layer 0: Width Theory} \quad
    Width $w$, phase transition, VC--width separation
  };
  
  \node[layer, fill=green!12, minimum width=11cm] (L1) at (0, 1.5) {%
    \textbf{Layer 1: Width Estimation} \quad
    CS operator, split--merge, structural ERM
  };
  
  \node[layer, fill=yellow!15, minimum width=10cm] (L2) at (0, 3.0) {%
    \textbf{Layer 2: Metric Slingshot} \quad
    $d_X \to d_Z{=}2$, grid cells, CLS mapping
  };
  
  
  
  
  \node[font=\scriptsize, rotate=90, gray] at (7.8, 2.0) {increasing generality $\longrightarrow$};
  
  
\end{tikzpicture}
\caption{Layer structure of StrLT. Each layer builds on the ones below it. Layer~0 defines width and the phase transition. Layer~1 makes width estimable. Layer~2 reduces the geometric cost. 
Classical SLT governs only the within-cell component at Layer~0; StrLT encompasses the full stack.}
\label{fig:layer-structure}
\end{figure}
}

\newcommand{\figComplexityClasses}{%
\begin{figure}[t]
\centering
\begin{tikzpicture}[scale=0.85, every node/.style={font=\small}]
  \draw[thick, purple!60, fill=purple!5, rounded corners=15pt] 
    (-5.0, -2.5) rectangle (5.0, 3.5);
  \node[purple!70, font=\bfseries] at (3.8, 3.1) {S-OPEN};
  
  \draw[thick, blue!60, fill=blue!5, rounded corners=12pt] 
    (-4.2, -2.0) rectangle (3.5, 2.8);
  \node[blue!70, font=\bfseries] at (2.4, 2.4) {S-TRACK};
  
  \draw[thick, green!60!black, fill=green!8, rounded corners=8pt] 
    (-3.4, -1.4) rectangle (2.0, 2.0);
  \node[green!60!black, font=\bfseries] at (0.8, 1.6) {S-LEARN};
  
  \draw[thick, red!60, fill=red!5, rounded corners=12pt, fill opacity=0.3] 
    (1.5, -2.2) rectangle (4.7, 1.5);
  \node[red!70, font=\bfseries] at (3.8, 1.1) {S-HARD};
  
  \node[font=\scriptsize, text width=3.5cm, align=center, green!40!black] at (-0.7, 0.2) {%
    Fixed width $w$\\
    Finite $d_{\mathrm{str}}$, $\mathrm{Pdim}(\mathcal{G})$\\
    Uniform gap $\eta^\star > 0$
  };
  
  \node[font=\scriptsize, text width=2.5cm, align=center, blue!60] at (-3.0, 2.5) {%
    $|w_{t+1}-w_t| \le 1$\\
    Tracking with\\
    bounded delay
  };
  
  \node[font=\scriptsize, text width=2.5cm, align=center, purple!60] at (-3.8, -1.7) {%
    $w_t \to \infty$\\
    Unbounded growth
  };
  
  \node[font=\scriptsize, text width=2.2cm, align=center, red!60] at (3.5, -0.5) {%
    $d_{\mathrm{str}} = \infty$\\
    or $\eta^\star \to 0$\\
    or super-poly.
  };
  
  \node[font=\scriptsize, text width=2cm, align=center, gray] at (2.5, -1.8) {%
    $\mathrm{S\text{-}OPEN} \cap \mathrm{S\text{-}HARD}$\\$\neq \varnothing$
  };
\end{tikzpicture}
\caption{Structural complexity classes. S-LEARN (fixed width) is strictly contained in S-TRACK (bounded width drift), which is strictly contained in S-OPEN (unbounded width growth). S-HARD (structurally intractable) overlaps with S-OPEN: some open-ended problems are also hard. The inclusions \(\mathrm{S\text{-}LEARN} \subsetneq \mathrm{S\text{-}TRACK} \subsetneq \mathrm{S\text{-}OPEN}\) are proved in Theorem~\ref{thm:slearn-strict-strack} and Proposition~\ref{prop:strack-strict-sopen}.}
\label{fig:complexity-classes}
\end{figure}
}

\newcommand{\figDecomposition}{%
\begin{figure}[t]
\centering
\begin{tikzpicture}[scale=0.85, every node/.style={font=\small}]
  \draw[thick] (0,4.5) -- (10,4.5);
  \node[font=\bfseries, above] at (5,4.5) {Total structural complexity $\mathfrak{R}_n^S$};
  
  \node[font=\Large] at (5,3.7) {$=$};
  
  \draw[thick, fill=red!15, rounded corners=3pt] (0.2, 2.2) rectangle (4.3, 3.2);
  \node[text width=3.8cm, align=center] at (2.25, 2.7) {%
    \textbf{Trap term}\\[-2pt]
    {\footnotesize $\sqrt{d_{\mathrm{str}}\log n / n}$}
  };
  
  \node[font=\Large] at (5, 2.7) {$+$};
  
  \draw[thick, fill=blue!15, rounded corners=3pt] (5.7, 2.2) rectangle (9.8, 3.2);
  \node[text width=3.8cm, align=center] at (7.75, 2.7) {%
    \textbf{Funnel term}\\[-2pt]
    {\footnotesize $LK \cdot \mathfrak{R}_n(\mathcal{G})$}
  };
  
  \draw[->, gray] (2.25, 2.0) -- (2.25, 1.0);
  \node[font=\scriptsize, text width=3.5cm, align=center, gray] at (2.25, -0.3) {%
    Controlled by\\
    structural graph dimension\\
    $d_{\mathrm{str}}(K,\delta,\mathcal{G})$\\[3pt]
    \emph{Structural Learning Theory}
  };
  
  \draw[->, gray] (7.75, 2.0) -- (7.75, 1.0);
  \node[font=\scriptsize, text width=3.5cm, align=center, gray] at (7.75, -0.3) {%
    Controlled by\\
    pseudo-dimension\\
    $\mathrm{Pdim}(\mathcal{G})$\\[3pt]
    \emph{Statistical Learning Theory}
  };
  
  \node[font=\scriptsize, text width=4cm, align=center, red!60!black] at (5, 1.5) {%
    No cross-term: trap and funnel\\
    contribute independently
  };
\end{tikzpicture}
\caption{The decomposition theorem (Theorem~\ref{thm:structural-decomposition}): structural Rademacher complexity decomposes additively into a trap term governed by the structural graph dimension and a funnel term governed by the predictor class Rademacher complexity. There is no cross-term, formalizing the Metric--Topology Factorization at the complexity-theoretic level.}
\label{fig:decomposition}
\end{figure}
}

\begin{abstract}
Learning in structured, multi-context, or non-stationary environments involves two orthogonal difficulties. The first is \emph{metric}: once the correct context is known, how hard is prediction within it? This is the domain of Statistical Learning Theory (SLT). The second is \emph{structural}: how many local contexts are required, and how can they be discovered from data? This paper develops \emph{Structural Learning Theory} (StrLT) for the structural axis.
We introduce \emph{width}, the minimum number of jointly contractive and low-risk cells needed to cover a learning problem. Width is incomparable with VC dimension: either can diverge while the other remains bounded. We show that width induces a \emph{phase transition}: if the allocated number of cells \(K<w\), learning suffers an irreducible structural error floor; if \(K\ge w\), the problem reduces to ordinary within-cell statistical learning. To estimate width, we introduce the \emph{contractive-similarity} (CS) operator, a task-adaptive graph kernel combining geometric locality with predictive compatibility. Its CS Laplacian exposes contractive basins through spectral separation. We further develop the \emph{metric slingshot}, which reuses low-dimensional latent contraction maps to reduce funnel-learning cost.
Together, width, CS estimation, and the slingshot decompose learning into trap discovery and funnel generalization, with deep implications for continual and lifelong learning in an open-ended environment.
\end{abstract}


\section{Introduction}
\label{sec:introduction}

\subsection{Motivation: the Two Complexities of Learning}
\label{subsec:intro_motivation}

Consider a mobile robot operating in a building composed of multiple rooms, where each room obeys a different local dynamics law. In one room, the floor is slippery, in another, friction is high, and in a third, the mapping from motor command to displacement is perturbed by a local magnetic field. From the standpoint of prediction and control \citep{spong2020robot}, the robot faces two logically distinct problems. First, it must determine \emph{which room it is in}, which is a discrete, structural, and essentially combinatorial problem \citep{russell1995modern}: an incorrect room assignment routes the observation to the wrong local model, and no amount of gradient refinement within that local model can repair the error. Second, once the correct room has been identified, the robot must learn \emph{how to act within that room}, which is a continuous, metric problem, typically well described by local smoothness, approximation, and generalization arguments of the type developed in statistical learning theory \citep{vapnik2013nature}.

We will refer to these two problems as the \emph{trap} and the \emph{funnel} \citep{lindgren1983multiple}. The trap is the problem of structural localization: which latent context or basin governs the current input? The funnel is the problem of local prediction once the correct context is known. Note that the trap-funnel distinction is not peculiar to robotics. In continual learning \citep{wang_comprehensive_2024}, the learner must decide whether a new sample belongs to an existing task or signals a new task. In multi-task learning \citep{zhang2021survey}, the learner must infer task identity or task membership before exploiting shared structure. In reinforcement learning \citep{sutton1998reinforcement}, the classical exploration-exploitation dichotomy (a.k.a. stability-plasticity dilemma \citep{kim2023stability}) has an analogous geometric reading: exploration is partly the discovery of the correct basin of behavior, while exploitation is optimization within that basin. In biological learning \citep{mcclelland1995there}, the same distinction appears in the complementary learning systems picture, where fast contextual indexing and slow within-context optimization are functionally separated.

The central claim of this paper is that these two forms of difficulty are governed by two \emph{orthogonal} complexity axes. The familiar axis is \emph{metric and continuous}: it measures how hard it is to learn \emph{within} a fixed context. Vapnik's Statistical Learning Theory (SLT) \citep{vapnik2013nature}, through VC dimension \citep{devroye2013probabilistic}, Rademacher complexity \citep{bartlett2002rademacher}, uniform convergence \citep{vapnik1971uniform}, and PAC learnability \citep{valiant1984theory}, gives a deep and essentially complete account of this axis. The missing axis is structural and topological: it measures how hard it is to discover \emph{which} context applies. Our goal is to formalize this second axis and to show that it cannot be reduced to the first.
The organizing principle is \emph{Metric-Topology Factorization} (MTF).

\figTrapFunnel

A learning problem decomposes into a topological indexing problem and a metric prediction problem (Fig. \ref{fig:trap-funnel}). The topological component determines which local basin the input belongs to; the metric component determines how to predict within that basin. The main complexity measure on the structural side is \emph{width} \citep{li2026local}: the minimum number of locally contractive cells required to cover the problem. Informally, width counts how many distinct ``contexts'' are needed before local learning becomes possible. As we will show, width behaves differently from VC dimension. The central negative result of the paper proves that width and VC dimension are fully \emph{orthogonal}: there are families of problems whose width diverges while VC dimension remains bounded, and families whose VC dimension diverges while width remains one.
The width-VC orthogonality has a direct conceptual implication - e.g., scaling model capacity \citep{snell2024scaling}, increasing the richness of the predictor class \citep{jacobs1991adaptive}, or refining standard statistical bounds can improve only the funnel. None of these operations can remove a structural deficiency on the trap side. If the learner allocates fewer contexts than the problem requires, then there is an irreducible structural error floor that does not vanish with more data. This is the phase transition at the heart of width theory \citep{nakahara_geometry_2018}.

\subsection{Relationship to existing theory}
\label{subsec:intro_related}

The natural point of departure is SLT \citep{vapnik2013nature,vapnik1971uniform}. SLT characterizes learnability within a fixed hypothesis class by combinatorial and empirical-process quantities such as VC dimension, growth functions, covering numbers, and Rademacher complexity \citep{shalev2014understanding,bartlett2002rademacher,koltchinskii2002rademacher,anthony1999learning}. These notions control approximation and estimation once the ambient representational geometry is fixed. In the language developed here, SLT governs the funnel. What it does not address is the structural question of how many local basins are required before a problem becomes metrically learnable. The present paper does not compete with SLT; it complements it. Our thesis is that classical SLT is exactly the right theory for one axis of learning, and silent on the other.
The multi-task and meta-learning literatures \citep{baxter2000model,maurer2016benefit} study how shared representations can reduce sample complexity across tasks. In StrLT terms, these works address the within-cell structure of the funnel: they show how shared features can amortize the cost of learning local predictors. StrLT complements this by addressing the prior question of how many distinct tasks (cells) must be distinguished. The finite mixture model literature \citep{mclachlan2000finite} is also adjacent: width can be seen as a structural generalization of the number of mixture components, but with the additional constraint that each component must be jointly contractive and feasible.
In reinforcement learning, the classical exploration-exploitation dichotomy has a structural counterpart in the options and temporal-abstraction framework \citep{sutton1999between}, where the agent must discover which high-level behavioral mode to invoke before optimizing within it. The sample complexity of reinforcement learning \citep{kakade2003sample} is typically analyzed within a fixed Markov Decision Process (MDP) \cite{russell1995modern}; StrLT addresses the structural cost of discovering which MDP, or which region of state space, the agent currently occupies.

Architecturally, the closest machine learning relatives are modular methods such as mixture-of-experts and sparse expert routing \citep{jacobs1991adaptive,shazeer2017outrageously,fedus2022switch}. These methods assume, explicitly or implicitly, that one should allocate different submodels to different regimes of the input space. In our notation, the number of experts \(K\) is the structural budget. StrLT predicts a simple and falsifiable phenomenon: modularity helps only once \(K\) reaches the true width \(w\). If \(K<w\), then the model is structurally underparameterized, and more local metric capacity cannot repair that mismatch.
There is also a meaningful contrast with domain adaptation \citep{ben2010theory}. Domain adaptation studies how to transfer between related distributions, typically under discrepancy or divergence assumptions that preserve some common structure. Instead, StrLT studies the more primitive question of how many genuinely distinct local regimes must be represented in the first place. It is therefore concerned not merely with shifted domains, but with the combinatorics of context discovery itself. The related literature on domain generalization \citep{gulrajani2020search} shares with StrLT the concern for robustness across environments, but typically seeks a single invariant predictor rather than the minimal partition into context-specific predictors that width formalizes.

The empirical literature on continual and lifelong learning has long emphasized that non-stationary environments require context-sensitive adaptation \citep{ring1994continual,thrun1998lifelong}. Catastrophic forgetting \citep{kirkpatrick2017overcoming,zenke2017continual}, task inference, task switching, and domain routing are central phenomena in that literature \citep{parisi2019continual}. Yet the field lacks a foundational complexity theory analogous to VC dimension on the structural side. Existing work is algorithmic, empirical, or architectural; it does not provide a canonical measure of structural difficulty. Structural Learning Theory is intended to fill that gap.
The neuroscience literature provides a striking qualitative parallel. Complementary Learning Systems (CLS) \citep{mcclelland1995there,kumaran2016learning,o2014complementary} posits a division between fast hippocampal encoding and slower neocortical integration. The hippocampal system rapidly distinguishes episodes and contexts; cortical systems gradually learn structured regularities. That is, CLS already contains a qualitative trap-funnel decomposition. What has been missing is a mathematical framework that makes this decomposition precise, quantifiable, and testable. One aim of the present theory is to supply such a framework.

\figLayerStructure

\subsection{Contributions and Paper Structure}
\label{subsec:intro_contrib}

The paper develops StrLT in layers (Fig. \ref{fig:layer-structure}), progressing from a structural complexity measure to estimation, reduction, learnability, and dynamic regimes. The main results may be summarized as follows.
Layer 0 introduces \emph{width}, the minimum size of a joint \((\gamma,\delta)\)-feasible contractive cover. Width is the fundamental structural complexity parameter of the theory. We prove a strict hierarchy theorem identifying width with the first Betti number on bouquet-of-circles constructions; a topology-geometry scaling law showing that width grows with both loop count and loop length relative to contraction scale; a VC-width separation theorem establishing that width and VC dimension are orthogonal; a structural sample-complexity lower bound of order \(\Omega(w\log w)\); and a phase transition theorem showing that when the number of cells \(K\) is below the true width \(w\), an irreducible structural gap \(\eta(w,K)>0\) remains no matter how much data is available.
Layer 1 develops the estimation theory. We introduce the contractive-similarity operator, show that it amplifies the spectral gap associated with contractive cells, prove local stability and uniform convergence of width estimation, and analyze a split-merge algorithm whose global convergence is proved under mild binary-dominant spectral regularity conditions. We also show that a correctly penalized structural ERM procedure consistently recovers width.
Layer 2 introduces the \emph{metric slingshot}, a navigation-inspired embedding architecture that reduces the effective geometric burden of width estimation. The slingshot transfers structural complexity into a low-dimensional metric coordinate system and gives a concrete bridge to biological representations such as grid cells and place cells.
Table~\ref{tab:slt-strlt-analogy} summarizes the analogy with Statistical Learning Theory that guides the paper.

\begin{table}[t]
\centering
\caption{Analogy between SLT and StrLT.  SLT controls the \emph{funnel}: prediction within a fixed hypothesis class.  StrLT controls the \emph{trap}: discovery, estimation, and reuse of contractive structural basins.}
\label{tab:slt-strlt-analogy}
\resizebox{\linewidth}{!}{
\begin{tabular}{lll}
\toprule
SLT object & StrLT analog & Section \\
\midrule
Hypothesis class \(\mathcal H\) 
& Structural problem \(P\) with contractive cells 
& \S\ref{sec:width-theory} \\

VC dimension \(d\) 
& Width \(w(P;\gamma,\delta)\) 
& \S\ref{sec:width-theory} \\

Empirical risk \(\widehat R_n(f)\) 
& CS width estimate \(\widehat w_n(G)\) 
& \S\ref{sec:width-estimation} \\

Risk minimization 
& Split--merge width estimation 
& \S\ref{sec:width-estimation} \\

Kernel / Gram matrix 
& Contractive-similarity kernel \(W^{\mathrm{CS}}(G)\) 
& \S\ref{sec:width-estimation} \\

Graph Laplacian / spectral regularization 
& CS Laplacian \(L_{\mathrm{CS}}(G)\) for basin discovery 
& \S\ref{sec:width-estimation} \\

Uniform convergence of empirical risk 
& Uniform convergence of CS width estimation 
& \S\ref{sec:width-estimation} \\

Model selection / SRM penalties 
& Penalized structural ERM over \(K\) contexts 
& \S\ref{sec:width-estimation} \\

Feature map / representation 
& Metric slingshot \(\phi:X\to Z\) 
& \S\ref{sec:metric-slingshot} \\

Local hypothesis class 
& Local funnel \(\pi_c\circ G_c^0\circ\phi\) 
& \S\ref{sec:metric-slingshot} \\

Capacity control within a class 
& Contraction transfer \(L_{\pi,c}\gamma_Z L_\phi(U)<1\) 
& \S\ref{sec:metric-slingshot} \\

Classical generalization rate 
& Per-funnel generalization after trap resolution 
& \S\ref{sec:metric-slingshot} \\




\bottomrule
\end{tabular}
}
\end{table}

\section{Foundational Layer: Width Theory}
\label{sec:width-theory}

This section develops the foundational layer of structural learning theory. We first formalize the metric-topology factorization, then define width, establish its basic properties, and prove the five core results within two categories: 1) \emph{system architecture} - the strict hierarchy theorem, the VC-width separation theorem, and the phase transition theorem; 2) \emph{computational complexity} - the topology-geometry scaling law and the structural sample-complexity lower bound.

\paragraph{Notation}
Throughout, \((X,d_X)\) denotes a compact metric input space, \(Y\) an output space, and \(P\) a distribution on \(X\times Y\). The loss function is
$\ell:\widehat Y\times Y\to[0,1]$,
assumed measurable and \(L\)-Lipschitz in its first argument. The predictor class is \(\mathcal G\), with generic element \(g:X\to\mathbb R\). The pseudo-dimension of \(\mathcal G\), denoted \(\mathrm{Pdim}(\mathcal G)\), is the largest \(d\) such that there exist points \(x_1,\dots,x_d\) and thresholds \(t_1,\dots,t_d\) for which every binary labeling \((b_1,\dots,b_d)\in\{0,1\}^d\) is realized by some \(g\in\mathcal G\) via \(g(x_i)\ge t_i \iff b_i=1\); it is the real-valued generalization of VC dimension and governs the metric (within-cell) sample complexity throughout the paper. The structural thresholds are \(\gamma\in(0,1)\) for contraction and \(\delta>0\) for risk.
Given a measurable set \(U\subseteq X\) and predictor \(g\in\mathcal G\), we write
$\kappa(g,U)
:=
\sup_{\substack{x,x'\in U\\x\neq x'}}
\frac{|g(x)-g(x')|}{d_X(x,x')}$
for the Lipschitz modulus of \(g\) on \(U\). A partition into \(K\) cells is denoted
$\Pi_K = \{U_1,\dots,U_K\}$.
The width of \(P\) at thresholds \((\gamma,\delta)\) is denoted
$w(P;\gamma,\delta)$,
or simply \(w(P)\) when \((\gamma,\delta)\) is fixed throughout. We will introduce the structural gap as \(\eta(w,K)\) in Theorem~\ref{thm:phase-transition}.
Later sections introduce the structural assignment class \(H_{K,\delta}(\mathcal G)\), the structural graph dimension \(d_{\mathrm{str}}\), and the structural Rademacher complexity \(\mathfrak R_n^S\), which we forward-reference here because their introduction is motivated by the width theory \footnote{Unless otherwise stated, all logarithms are natural.}.

\paragraph{The Metric-Topology Factorization}

The starting point of the theory is the observation that prediction in a structured environment can be decomposed into \emph{context identification} and \emph{local prediction}, which we formalize at the levels of both system architecture and computational complexity as shown in Table \ref{tab:decomp}.

\begin{table}[t]
\centering
\small
\caption{The five Layer~0 results decomposed by level (architecture vs.\
complexity) and axis (trap vs.\ funnel). The total sample complexity
$n_{\mathrm{total}} = \Omega(w\log w) + w \cdot O(\mathrm{Pdim}(\mathcal{G})/\varepsilon^2)$
reflects the additive decomposition: no cross-term.}
\label{tab:layer0-decomposition}
\begin{tabular}{@{}p{2.2cm}p{5.2cm}p{5.2cm}@{}}
\toprule
& \textbf{Context identification (trap)} & \textbf{Local prediction (funnel)} \\
\midrule
\textbf{Architecture} \newline\scriptsize\emph{What the system must be}
& \textbf{Strict hierarchy:} problems of width $w = 1, 2, 3, \dots$ exist
  with no collapse; the architecture must support a variable number of
  structural slots. \newline
  \textbf{Phase transition} ($K \ge w$ is necessary): fewer than $w$ slots
  $\Rightarrow$ irreducible error floor $\eta(w,K) > 0$.
& \textbf{VC-width separation:} metric capacity (VC dim, parameters,
  depth) is orthogonal to structural capacity; the two subsystems must
  be architecturally independent. \newline
  \textbf{Phase transition} ($K \ge w$ is sufficient): once $K \ge w$,
  standard SLT convergence applies within each cell. \\
\midrule
\textbf{Complexity} \newline\scriptsize\emph{What the system must pay}
& \textbf{Scaling law:} $w \sim L/D_0 \times \text{topology}$; structural
  complexity scales with the ratio of problem diameter to contraction
  scale, modulated by Betti numbers. \newline
  \textbf{Sample cost:} $\Omega(w \log w)$ observations needed to
  identify $w$ contexts (coupon-collector bound).
& \textbf{Scaling law:} (governed by VC dimension/Pdim and Rademacher complexity in conventional SLT). \newline \textbf{Sample cost:} $w \cdot O(\mathrm{Pdim}(\mathcal{G})/\varepsilon^2)$
  observations required for $\varepsilon$-accurate prediction across all
  $w$ cells (standard uniform convergence per cell). \\
\bottomrule
\end{tabular}
\label{tab:decomp}
\end{table}

\begin{definition}[Metric-Topology Factorization]
\label{def:mtf}
A learning system for a problem \(P\) on \(X\times Y\) consists of:
1) a \emph{topological indexer}
$\Sigma:X\to \Delta([K])$,
mapping each input to a distribution over \(K\) contexts;
2) a family of \emph{metric learners}
$\{G_c\}_{c=1}^K,\quad G_c:X\to \widehat Y$,
one predictor per context;
3) a prediction rule that routes \(x\) through the selected context:
$\widehat y = G_{\Sigma(x)}(x)$,
where the notation suppresses either hard routing or sampling from \(\Sigma(x)\).
\end{definition}

\begin{definition}[Structural decoupling]
\label{def:structural-decoupling}
Let \(L_{\mathrm{trap}}\) denote the loss associated with context identification and \(L_{\mathrm{funnel}}\) the within-context prediction loss. The system is \emph{structurally decoupled} if
$\nabla_{\theta_\Sigma}L_{\mathrm{funnel}}=0
\quad\text{and}\quad
\nabla_{\theta_G}L_{\mathrm{trap}}=0$,
where \(\theta_\Sigma\) are the parameters of the indexer and \(\theta_G\) the parameters of the local predictors.
\end{definition}
\noindent
Structural decoupling expresses the idealized principle that trap and funnel should not interfere at the gradient level. In practice, many machine learning architectures violate this principle by forcing shared parameters to absorb both structural and metric burden \citep{shalev2014understanding}. The point of Definition~\ref{def:structural-decoupling} is not that every useful system must satisfy exact gradient orthogonality, but that the two roles are mathematically distinct and should be treated as such in the complexity theory \citep{sipser_introduction_1996}. The structural unit of the theory is a set on which local prediction is simultaneously geometrically contractive and statistically accurate.
The remainder of this section develops the complexity measure attached to the trap component.


\begin{definition}[\((\gamma,\delta)\)-contractive set]
\label{def:contractive-set}
Let \(\gamma\in(0,1)\) and \(\delta>0\). A measurable set \(U\subseteq X\) is \emph{\((\gamma,\delta)\)-contractive} for problem \(P\) with respect to predictor class \(\mathcal G\) if there exists \(g\in\mathcal G\) such that
$\kappa(g,U)\le \gamma
\quad\text{and}\quad
\mathbb E[\ell(g(X),Y)\mid X\in U]\le \delta$,
whenever \(P(X\in U)>0\).
\end{definition}
\noindent
Definition~\ref{def:contractive-set} uses the joint \((\gamma,\delta)\)-feasibility notion that will later allow the structural theory of width to interface cleanly with the risk-based learning theory (an accompanying paper still under development). The contraction condition \(\kappa(g,U)\le\gamma\) and the risk condition \(\mathbb E[\ell]\le\delta\) are the geometric and statistical requirements. Neither implies the other in general, but both are needed for the full theory of StrLT.

\paragraph{Width: definition and basic properties}

Urysohn's Lemma \citep{munkres2025elements} states that in a normal
topological space, any two disjoint closed sets \(A\) and \(B\) can be
separated by a continuous function \(f: X \to [0,1]\) with \(f|_A = 0\)
and \(f|_B = 1\). The lemma guarantees the \emph{existence} of a
separator but says nothing about how to construct one or how smooth it
can be. Its proof proceeds by constructing, for each dyadic rational
\(r \in [0,1]\), an open set \(U_r\) satisfying a nested chain of
inclusions \(A \subset U_r \subset \overline{U_r} \subset U_s \subset B^c\)
whenever \(r < s\), and then defining \(f\) as the infimum over these
inclusions. The construction succeeds precisely because the two closed
sets are disjoint (there is topological room between them for the nested
open sets to fit). 
Width arises from asking: \emph{what if the closed sets are not disjoint?}
In a multi-regime learning problem \citep{osmani2026general}, the supports of distinct basins
\(B_1, \dots, B_w\) may overlap in the input space or at least may not
be separated by any margin in the original metric. A single Urysohn
separator cannot be constructed because the disjointness precondition
fails. The learner's task is to \emph{manufacture} the disjointness that
Urysohn's Lemma requires by partitioning the input space into cells
within which the prediction problem is locally coherent. Each cell
\(U_k\) is a region where one predictor \(g_k\) is
\((\gamma,\delta)\)-feasible, meaning the cell is metrically contracted
and the prediction risk is controlled. Once such a partition is found,
the supports of different regimes are separated into different cells, and
a Urysohn-type separator \emph{between cells} becomes constructible.

\begin{definition}[Width]
\label{def:width}
Fix \((\gamma,\delta)\). The \emph{width} of problem \(P\) is
$w(P;\gamma,\delta)
:=
\min\Bigl\{
K\in\mathbb N:
\exists\ \text{open cover }\{U_1,\dots,U_K\}\text{ of }X
\text{ with each }U_k\text{ \((\gamma,\delta)\)-contractive}
\Bigr\}$.
When \((\gamma,\delta)\) is fixed, we abbreviate \(w(P;\gamma,\delta)\)
to \(w(P)\) or \(w\).
\end{definition}

\noindent
The use of an open cover rather than a measurable partition is
deliberate. It makes the definition topologically natural and allows
width to interact cleanly with standard covering arguments
\citep{conway2013sphere}. 
Width is the minimum number of local Urysohn constructions
needed to resolve structural conflicts in the problem. Each cell
\(U_k\) is a region where the Urysohn precondition is locally
satisfied: within \(U_k\), the prediction landscape has a single basin,
so there is nothing to separate. The structural difficulty is entirely
in determining how many such regions are needed and where their
boundaries lie. A problem of width \(w = 1\) satisfies the Urysohn
precondition globally: the entire input space is one basin and a
single continuous predictor suffices. A problem of width \(w > 1\)
requires the learner to partition the space into \(w\) regions, each
locally Urysohn-compatible, before prediction becomes well-posed.
The Urysohn perspective clarifies why width is a topological invariant because width counts the
number of topological obstructions to single-predictor coverage (the
number of independent conflicts that must be resolved by partitioning).
The metric enters only through the feasibility thresholds
\((\gamma, \delta)\), which determines how much contraction and how
little risk each cell must achieve. In summary, width itself is most naturally a cover quantity with the following basic monotonicity and finiteness properties.

\begin{proposition}[Basic properties of width]
\label{prop:width-basic}
Let \(P\) be a problem on compact metric space \((X,d_X)\), and let \(\mathcal G\) be a predictor class.
(1) If \(\gamma_1\le \gamma_2\) and \(\delta_1\le \delta_2\), then
$w(P;\gamma_2,\delta_2)\le w(P;\gamma_1,\delta_1)$.
(2) If \(X\) is compact and there exists a finite open cover of \(X\) by sets on which the target admits continuous local predictors in \(\mathcal G\), then \(w(P;\gamma,\delta)<\infty\) for sufficiently relaxed \((\gamma,\delta)\).
(3) \(w(P;\gamma,\delta)=1\) if and only if the entire problem admits a global \((\gamma,\delta)\)-contractive predictor.
\end{proposition}


\subsection{System Architecture Results}

\paragraph{The Strict Hierarchy Theorem}

The first theorem identifies width with a topological invariant on a canonical family of examples \citep{hatcher2002algebraic}.

\begin{theorem}[Strict Hierarchy]
\label{thm:strict-hierarchy}
Let \(M\) be a bouquet of \(w\) circles, i.e.,
$M = \bigvee_{j=1}^w S^1_j$,
with common identification point \(x_0\). Let \(P\) be a learning problem on \(M\) such that:
1) for each circle \(C_j\), there exists a predictor \(g_j\in\mathcal G\) making a neighborhood of \(C_j\) \((\gamma,\delta)\)-contractive;
2) no predictor in \(\mathcal G\) is simultaneously \((\gamma,\delta)\)-contractive on any open set that intersects two distinct circles \(C_a\) and \(C_b\) in neighborhoods of \(x_0\).
Then
$w(P;\gamma,\delta)=w=\beta_1(M)$,
where \(\beta_1(M)\) is the first Betti number of \(M\).
\end{theorem}

\noindent
\begin{remark}
Theorem~\ref{thm:strict-hierarchy} shows that width can coincide exactly with a topological invariant of the underlying problem manifold. Each independent loop that forces its own local predictor contributes one unit of width (Fig. \ref{fig:bouquet}). Note that such topological complexity (\(\beta_1(M)\)) is blind to the graph Laplacian operator in spectral clustering \citep{ng2001spectral} because Laplacian can only read off the zeroth Betti number (the connected component).
\end{remark}

\figBouquet

\paragraph{The VC-width separation theorem}
Next, we show how width is related to the classical complexity measure of VC dimension \citep{vapnik1971uniform}. Width can be arbitrarily large while VC dimension remains bounded, and VC dimension can be arbitrarily large while width remains one; therefore, width and VC dimension are fully orthogonal complexity measures.

\begin{theorem}[VC-width separation]
\label{thm:vc-width-separation}
Width and VC dimension are orthogonal complexity measures:
(1) There exists a sequence of problems \(\{P_w\}\) with \(w(P_w)\to\infty\) while \(\operatorname{VC}(\mathcal G)=O(1)\);
(2) There exists a sequence of classes \(\{\mathcal G_d\}\) and problems \(\{Q_d\}\) with \(\operatorname{VC}(\mathcal G_d)\to\infty\) while \(w(Q_d)=1\).
\end{theorem}

\noindent
\emph{Corollary: scaling cannot repair structural deficiency.}
Increasing the richness of \(\mathcal G\) increases VC dimension but does not affect width. In other words, scaling improves the funnel continuously \citep{snell2024scaling}; touching the trap requires crossing a discrete phase transition. Since no continuous operation can produce a discrete jump, the phase transition becomes a mathematical inevitability for any system that addresses structural complexity. 

\paragraph{The phase transition at the true width.}
The orthogonality results above establish that width \emph{exists} as an
independent complexity axis. We now show that it has sharp operational
consequences: a discrete phase transition divides the learner's behavior
into two qualitatively different regimes. When the structural budget
$K$ meets or exceeds the true width $w$, each basin can be assigned its
own cell and standard metric learning takes over - more data yields
better predictions at the usual $O(1/\sqrt{n})$ rate. When $K < w$, the
pigeonhole principle forces at least one cell to mix points from two
incompatible basins, and no within-cell predictor, however
powerful, can serve both basins well. The resulting error floor is not a
finite-sample artifact or an optimization failure; it is a
population-level obstruction that persists for all $n$. To state this
precisely, we need a mild non-degeneracy condition ensuring that mixing
two basins incurs a detectable cost.

\begin{assumption}[Structural non-degeneracy]
\label{ass:nondegeneracy}
There exists \(\eta(w,K)>0\) for every \(K<w\) such that if a \(K\)-cell cover forces at least one cell to mix points from two distinct true basins, then every predictor on that cell incurs excess risk at least \(\eta(w,K)\).
\end{assumption}

\begin{theorem}[Phase transition at \(K=w\)]
\label{thm:phase-transition}
Let \(w=w(P;\gamma,\delta)\) and let \(K\) be the structural budget.
(1) If \(K\ge w\), the generalization gap obeys the metric-rate bound
$\mathrm{Gap}(n;K) \le C\sqrt{\frac{K\log N(\mathcal Z_c,\varepsilon)+\log(K/\delta')}{n}}$.
(2) If \(K<w\), under Assumption~\ref{ass:nondegeneracy},
$\mathrm{Gap}(n;K)\ge \eta(w,K)>0$
for all \(n\).
\end{theorem}

\begin{remark}[Anatomy of the phase transition]
The two parts of the theorem describe qualitatively different regimes
separated by a sharp boundary.
\emph{Part (1)} is reassuring: once the structural budget meets the
true width, the problem reduces to $K$ independent within-cell learning
problems, each governed by standard uniform convergence. The covering
number $N(\mathcal{Z}_c, \varepsilon)$ and the $\sqrt{1/n}$ rate are
familiar from classical SLT where Vapnik's theory takes over in full. The
factor of $K$ reflects the union bound over cells and the splitting of
$n$ samples across $K$ groups. In this regime, more data always helps,
and scaling model capacity within each cell yields standard returns.
\emph{Part (2)} is the surprise: the error floor $\eta(w,K) > 0$ is a
\emph{population-level} quantity which does not depend on $n$, on the
algorithm, or on the predictor class. It depends only on the
relationship between $K$ and $w$. By pigeonhole, at least one of the
$K$ cells must contain points from two distinct basins. Every predictor
on that cell must compromise between the two basins, and the price of
that compromise is at least $\eta(w,K)$. A more powerful predictor
fits the compromise more precisely, but the compromise itself is the
error. This floor is the \emph{trap}: the irreducible cost of
structural under-resolution.
\end{remark}

Three features distinguish Theorem \ref{thm:phase-transition} from familiar learning-theoretic
lower bounds. First, the floor does not vanish with $n$ (it is not a
finite-sample artifact but a permanent obstruction). Second, the floor
does not depend on the complexity of $\mathcal{G}$ (it persists even
for predictor classes of infinite VC dimension). Third, the gap between
$K = w - 1$ and $K = w$ \emph{widens} as $n$ increases, because the
within-cell error (which shrinks with $n$) makes the structural floor
(which does not) proportionally more dominant. The gap widening is the
experimental signature of structural under-resolution and the opposite
of statistical overfitting, where gaps narrow with more data.
The transition is a quantitative form of Ashby's \emph{Law of Requisite
Variety} \citep{ashby1956introduction}: $K \ge w$ is the condition
under which the learner's structural capacity matches the
environment's structural complexity. Below the threshold, no amount
of metric improvement (e.g., more data/parameters/computation) can compensate. Structural
intervention (e.g, increasing $K$ from $w - 1$ to $w$) is needed.

\subsection{Computational Complexity Results}

The preceding results characterize the \emph{architecture} of structural
learning: what the system must be (a variable number of decoupled
structural slots) and when it succeeds or fails (the phase transition at
$K = w$). We now turn to the complementary question: \emph{what must the
system pay?} Two costs must be quantified. First, how large is $w$
itself? The width is defined abstractly as a minimum cover size, but a
useful theory must relate it to measurable geometric and topological
properties of the problem. Second, given that $w$ contexts exist, how
many observations are needed to discover all of them? The first question
is answered by the topology-geometry scaling law, which expresses width
as a function of the problem's Betti numbers (how many independent
topological obstructions exist) and its diameter-to-contraction ratio
(how geometrically spread out each obstruction is). The second is
answered by the structural sample-complexity lower bound, which
establishes that context discovery has an irreducible
$\Omega(w \log w)$ data cost - a coupon-collector price that no
algorithm can avoid \citep{neal2008generalised}. Together, these two results close the loop: the
scaling law tells you how structurally complex the problem is, and the
sample bound tells you how many observations that complexity demands.

\paragraph{Topology-geometry scaling}

The strict hierarchy theorem shows that topological complexity (measured by Betti numbers) drives width. Persistent topological structure in data \citep{carlsson2009topology} can be a source of structural learning difficulty. The next theorem adds geometry to the picture of scaling:

\begin{theorem}[Topology-geometry scaling]
\label{thm:topology-geometry}
Let \(M\) be a compact metric manifold containing \(\beta_1(M)\) independent fundamental cycles. Suppose each such cycle has length at least \(L\), and every \((\gamma,\delta)\)-contractive set has diameter at most \(D_0\). Then we have
$w(P;\gamma,\delta) \ge c\, \beta_1(M)\,\frac{L}{D_0}$
for a universal constant \(c>0\).
\end{theorem}

\begin{remark}[Width as a geometric refinement of topology]
The lower bound $w \ge c\,\beta_1(M)\,L/D_0$ reveals that width is
controlled by two factors that are conceptually independent.
The first factor, $\beta_1(M)$, is purely topological: it counts how
many independent loops in the input manifold force structurally
incompatible prediction regimes. This is the qualitative skeleton of
the structural problem, which tells you how many distinct obstructions
exist, regardless of their size.
The second factor, $L/D_0$, is purely geometric: it measures how many
contractive cells are needed to tile each loop. A loop of length $L$
covered by cells of diameter $D_0$ requires at least $L/D_0$ cells,
just as a rope of length $L$ requires at least $L/D_0$ patches of tape
of width $D_0$ to cover it. This factor is independent of topology because it
applies even to a single loop ($\beta_1 = 1$).
\end{remark}

The product structure of Theorem \ref{thm:topology-geometry} means that the width can be large for two entirely
different reasons. A problem with many short loops ($\beta_1$ large,
$L/D_0$ moderate) has high width because of \emph{topological richness} - many
independent structural conflicts, each individually easy to resolve. A
problem with a single long loop ($\beta_1 = 1$, $L/D_0$ large) has
high width because of \emph{geometric extent} - one structural conflict that
is spread across a large region of the input space and requires many
cells to tile. The two sources of difficulty are multiplicative, not
additive, which explains why the width is not a homological count. The Betti number
$\beta_1$ alone would assign width $1$ to a single loop regardless of
its length, missing the geometric cost entirely. Conversely, the
diameter ratio $L/D_0$ alone would miss the fact that two disjoint
short loops contribute independently to the width. Width captures both
dimensions simultaneously, making it a \emph{geometric refinement of
topological complexity} rather than either one in isolation.

\paragraph{A structural sample-complexity lower bound.}
The topology-geometry scaling law tells us how large $w$ is; we next ask how many
observations are needed to discover all $w$ basins. The answer is
strictly super-linear in $w$, because the learner must not only visit
each basin but also accumulate enough evidence to distinguish the last
basin from the others (a coupon-collector cost that no algorithm can
circumvent). In contrast to the sample complexity lower bound in SLT \citep{devroye2013probabilistic}, we have the following result for StrLT.

\begin{theorem}[Structural sample-complexity lower bound]
\label{thm:sample-lower-bound}
Any algorithm that identifies all \(w\) contractive basins with probability bounded away from zero requires \(\Omega(w\log w)\) samples in the worst case.
\end{theorem}

\begin{remark}[StrLT vs.\ SLT sample complexity]
The $\Omega(w\log w)$ lower bound is fundamentally different from the
classical statistical lower bound $\Omega(d/\varepsilon^2)$, where $d$
is the VC dimension or pseudo-dimension. The statistical bound measures
the cost of \emph{estimation}: how many samples are needed to learn a
good predictor once the prediction problem is fixed. It decreases
smoothly with tolerance $\varepsilon$ and is driven by the richness of
the function class. The structural bound measures the cost of
\emph{identification}: how many samples are needed to discover that $w$
distinct regimes exist and to observe at least one example from each.
The $\log w$ factor is a coupon-collector penalty, the expected wait
to see the last of $w$ equally likely regimes, and has no analog in
standard VC theory. Critically, the two costs are additive: the total
sample complexity is
$n_{\mathrm{total}} = \underbrace{\Omega(w\log w)}_{\text{identification}}
+ \underbrace{w \cdot O(\mathrm{Pdim}(\mathcal G)/\varepsilon^2)}_{\text{estimation}}$,
with no cross-term. A learner that has identified all $w$ basins still
faces the full statistical cost of fitting a predictor within each one;
a learner that has fitted a perfect predictor within $w-1$ basins has
made zero progress toward discovering the $w$-th. The two costs live
on orthogonal axes and must be paid independently.
\end{remark}

\section{Layer 1: Width Estimation via Urysohn Machine}
\label{sec:width-estimation}

Section~\ref{sec:width-theory} established width as the structural complexity parameter of learning. The main results there were population-level statements: width is defined through a contractive cover of the underlying problem, and the phase transition at \(K=w\) describes what happens if the learner allocates too few cells. The next question is \emph{algorithmic} and \emph{statistical}: given only finitely many samples, how can the width be estimated from data? We tackle this question in three steps: 1) Urysohn Machine construction, 2) Width Estimation via Contractive-Similarity (CS) Operator, and 3) Bilateral spectral clustering via CS Laplacian.

\subsection{Urysohn Machine Construction}
The width estimation problem has a natural computational interpretation through the Urysohn Machine (\textsf{UM}), a model of computation grounded in metric topology rather than symbolic manipulation \citep{sipser_introduction_1996}. The \textsf{UM} operates on a \emph{Metric Library}, a discrete metric space serving simultaneously as program and workspace, through a stack-based architecture whose fundamental unit is the \emph{Urysohn Triple}: a triple \(\tau = (\Sigma, \Pi, f)\) consisting of a support region \(\Sigma\) (where the triple is defined), a target partition \(\Pi\) (which concepts it discriminates), and a classifier \(f\) (mapping points to partition labels). Each Urysohn Triple instantiates Urysohn's Lemma \citep{urysohn_zum_1925} locally: it provides a concrete separating function for a specific partition within its support.
The correspondence with StrLT's width-realizing cover is direct. Each cell \(U_k\) in a width-realizing partition \(\Pi_w = \{U_1,\dots,U_w\}\) corresponds to a Urysohn Triple \(\tau_k = (U_k, \{B_j \cap U_k\}, g_k)\), where the support is the cell, the target partition is the restriction of the basin structure to that cell, and the classifier is the per-cell predictor. Width \(w\) is therefore the minimum number of Urysohn Triples needed to cover the input space such that each triple's classifier is \((\gamma,\delta)\)-feasible. The width estimation problem reduces to: \emph{construct the minimal Metric Library that resolves all structural boundaries in the data}.


\begin{definition}[Decision Boundary]\label{def:boundary}
  Given a Urysohn Triple $\tau = (\Sigma, \Pi, f)$, the \emph{decision boundary}
  $\partial\tau$ is the set of points in~$\Sigma$ at which~$f$ is discontinuous, or
  equivalently, the set of points whose every neighborhood contains points mapped to
  distinct labels. When~$\Sigma$ is embedded in a metric space, $\partial\tau$ inherits a
  natural measure of length (or more generally, $(n{-}1)$-dimensional Hausdorff measure).
\end{definition}

\begin{definition}[Decision-boundary Measure]\label{def:width}
  The \emph{decision-boundary measure} of a triple~$\tau$, denoted $\Width(\tau)$, is the total
  length (or Hausdorff measure) of its decision boundary~$\partial\tau$. For a collection
  of triples $T = \{\tau_1, \ldots, \tau_k\}$, the \emph{total measure of decision-boundary} is
  $\Width(T) = \sum_{i} \Width(\tau_i)$.
\end{definition}

Decision-boundary measure captures the intuition that some classification problems are intrinsically
harder than others, independent of the computational model used to solve them. A circle
with circumference~$2\pi r$ has a decision-boundary measure of~$2\pi r$. An Archimedean spiral with~$n$
turns has a decision-boundary measure that grows quadratically in~$n$. The spiral is intrinsically harder, not
because it requires a more clever program, but because its decision boundary is
longer (there is more topological structure to resolve).

\begin{theorem}[Amortized Separation Theorem]\label{thm:amortized}
  Let $\tau = (\Sigma, \Pi, f)$ be a Urysohn Triple with decision boundary measure of  $\Width(\tau)$. Let $\{B_1, \ldots, B_k\}$ be a collection of~$k$ basis
  functions (triples with simple supports) used to approximate~$f$ to accuracy~$\eps$
  on~$\Sigma$. Then $k \;\geq\; \frac{\Width(\tau)}{C \cdot \eps}\,$,
  where~$C$ is a constant depending only on the geometry of the ambient space and the
  Lipschitz constants of the basis functions.
\end{theorem}

The theorem has a natural interpretation: complex boundaries require proportionally more
computational resources, and no amount of clever programming can circumvent this lower
bound. This is the formal expression of intrinsic complexity.
A key consequence of the~\UM{} formalism concerns the role of metric contraction in
inference.

\begin{definition}[Class-aware contraction]\label{def:contraction}
  Let $(S,d)$ be a metric space with a classification $f\colon S \to \Pi$. A
  \emph{class-aware contraction} is a metric $d'\colon S \times S \to \R_{\geq 0}$ such
  that:
  \begin{enumerate}
    \item $d'(x,y) \leq \lambda \cdot d(x,y)$ for all $x,y$ with
      $f(x) = f(y)$, where $\lambda < 1$; and
    \item $d'(x,y) \geq d(x,y)$ for all $x,y$ with $f(x) \neq f(y)$.
  \end{enumerate}
\end{definition}

When the metric is contracted in this class-aware fashion, the topology of class regions
is preserved, but the metric budget for expressing that topology is reduced. More importantly,
the geodesic structure reshapes: paths that remain within a single class become metrically
shorter, while paths that cross boundaries become relatively longer.

\begin{theorem}[Geodesic Inference Bound]\label{thm:geodesic}
  Let $\Lib$ be a Metric Library with stack depth~$d$ and let $\lambda_i$ be the
  contraction factor at level~$i$. If the target classification has decision boundary measure of~$\Width$, then the inference path length under the contracted metric is at most
    $\sum_{i=1}^{d} \lambda_i \cdot \frac{\Width}{d}\,$,
  which for uniform contraction $\lambda_i = \lambda < 1$ yields total path length
  $\lambda \cdot \Width < \Width$.
\end{theorem}

The~\UM{}'s stack architecture provides the discrete realization of this principle. Each
level of the stack represents a different scale of contraction: the bottom triple operates
at maximum contraction (coarsest classification), and the top triple operates at minimum
contraction (finest classification). The total inference path length is bounded by the sum
$\sum_{i=1}^{d} \lambda_i \cdot \frac{\Width}{d}\,$, which is fundamentally shorter than ambient-space search through
the full boundary structure of length~$\Width$.

\begin{remark}[Anisotropic contraction]\label{rem:anisotropic}
The bound in \Cref{thm:geodesic} assumes isotropic contraction within each class
region. Anisotropic contractions, contracting strongly parallel to the decision
boundary and weakly perpendicular to it, can yield strictly better bounds. Contracting along the boundary shortens the boundary (reducing $\Width$) without reducing the separation between classes. We conjecture that the optimal contraction for a given classification problem is determined by the principal curvatures of the decision boundary.
\end{remark}

\subsection{Width Estimation via Contractive-Similarity (CS) Operator}

\paragraph{The width estimation problem}

With the computational framing of \UM in place, we formalize the width estimation problem next. Let
$\{(x_i,y_i)\}_{i=1}^n \overset{\mathrm{i.i.d.}}{\sim} P$
be an observed sample from the problem distribution on \(X\times Y\). The width estimation problem is to construct an estimator
$\widehat w_n$
such that
$\Pr(\widehat w_n = w(P;\gamma,\delta)) \to 1
\quad\text{as }n\to\infty$.
At a conceptual level, the problem decomposes into three subproblems.
First, for a given predictor \(G\), we need to determine how many \((\gamma,\delta)\)-contractive cells are supported by the joint geometry of \(X\) and the output structure induced by \(G\). This is the problem of \emph{spectral structural detection}.
Second, we need to optimize over predictors in order to approximate the predictor that realizes the smallest feasible cover. This is a uniform convergence problem over a predictor class.
Third, having computed empirical structural risks at multiple values of \(K\), we need to select the correct \(K\) (the structural model-selection problem).
The rest of this section treats these in order. The first two culminate in a sample-based width estimator \(\widehat w_n(G)\) for fixed \(G\), and the third culminates in a globally consistent estimator \(\widehat w_n\).

\paragraph{Why the graph Laplacian fails}

A first attempt is spectral clustering \citep{von2007tutorial}. Given the sample points \(\{x_i\}_{i=1}^n\), build a similarity graph using only spatial proximity, form the graph Laplacian, and count the number of near-zero eigenvalues. If the number of connected or weakly connected components corresponded to the number of contractive cells, this would estimate width. The problem is that width does \emph{not} count connected components. It counts contractive cells, or equivalently, the minimum number of Urysohn Triples. Connectivity and contractivity are fundamentally different notions. In the \textsf{UM} framework, the graph Laplacian builds a Metric Library that ignores the classifier: it uses only the support geometry \(\Sigma\) while discarding the partition \(\Pi\) and the classifier \(f\). A valid Urysohn Triple requires all three components. We formalize our observation into the following proposition.

\begin{proposition}[Laplacian blindness]
\label{prop:laplacian-blindness}
Let \(M\) be a bouquet of \(w\ge 2\) circles equipped with the natural geodesic metric and sampled densely near the common basepoint \(x_0\). Then the standard graph Laplacian on the sample cloud has exactly one near-zero eigenvalue, while the true width is \(w\). Consequently, the plain graph Laplacian is not a consistent estimator of width in general.
\end{proposition}


\paragraph{The Contractive-Similarity (CS) operator}

To overcome the blindness of the graph Laplacian, we propose an empirical operator that distinguishes geometric connectedness from task-compatible connectedness.
The key insight from the \UM is that classification requires a \emph{task-aware} or \emph{class-aware} metric: one that contracts within-class distances while preserving or expanding cross-class distances. The CS operator implements this principle statistically with a bilateral kernel design.

\begin{definition}[Contractive-Similarity kernel]
\label{def:cs-kernel}
Fix a predictor \(G:X\to\mathbb R\), a spatial scale \(r_x>0\), and an output scale \(\sigma_y>0\). The \emph{contractive-similarity} (CS) kernel is
$W_{ij}(G)
=
\mathbf 1[d_X(x_i,x_j)\le r_x]
\exp\!\left(
-\frac{|G(x_i)-G(x_j)|^2}{\sigma_y^2}
\right)$.
\end{definition}
\noindent
The kernel combines geometric locality with predictive compatibility (conceptually similar to bilateral filtering for image processing \citep{tomasi1998bilateral}): two points are treated as structurally similar only if they are both spatially close \emph{and} their predictions agree. Based on the CS kernel, we can generalize the graph Laplacian into CS Laplacian for the task of width estimation. 

\begin{definition}[CS Laplacian]
\label{def:cs-laplacian}
Given the CS kernel \(W(G)\), let \(D(G)_{ii} := \sum_{j=1}^n W_{ij}(G)\) be the degree matrix. The normalized CS Laplacian is
$L_{\mathrm{CS}}(G)
:=
I - D(G)^{-1/2} W(G) D(G)^{-1/2}$.
\end{definition}

\begin{definition}[CS width estimate]
\label{def:cs-width-estimate}
Fix a spectral threshold \(\tau_{\mathrm{spec}}>0\). The CS width estimate associated with predictor \(G\) is
$\widehat w_n(G)
:=
\#\{\lambda_k(L_{\mathrm{CS}}(G)) : \lambda_k \le \tau_{\mathrm{spec}}\}$.
\end{definition}
\noindent

The definitions above turn the abstract notion of a contractive cover into an empirical spectral object.  
If a proposed cell is genuinely contractive, nearby points inside the cell should have similar predictions, so the CS graph remains well-connected.  
If a cell actually mixes two or more structural basins, then points may still be close in \(X\), but their predictions differ across the basin boundary; the CS kernel therefore weakens precisely those edges that a purely geometric graph Laplacian would keep \citep{chung1997spectral}.  
Consequently, the spectrum of \(L_{\mathrm{CS}}(G)\) separates within-basin connectivity from cross-basin incompatibility: near-zero modes count task-compatible components, while the gap after these modes measures how sharply the basins are separated.  
The following theorem makes this mechanism quantitative by showing that predictive disagreement exponentially suppresses inter-cell conductance and thereby amplifies the spectral gap used by the width estimator.

\begin{theorem}[Spectral gap amplification]
\label{thm:spectral-gap-amplification}
Let \(\phi_{\mathrm{in}},\phi_{\mathrm{out}}\) are within-cell and inter-cell conductance quantities, and \(\delta_y\) denote the within-cell prediction variation and \(\Delta_y\) the cross-cell prediction gap. The CS kernel suppresses inter-cell conductance by the factor
$\exp\!\left(
-\frac{\Delta_y^2-\delta_y^2}{\sigma_y^2}
\right)$.
The effective spectral gap satisfies
$g_{\mathrm{eff}}
\ge
c_1\frac{\phi_{\mathrm{in}}^2}{w^4}
-
C_1
\exp\!\left(
-\frac{\Delta_y^2-\delta_y^2}{\sigma_y^2}
\right)\phi_{\mathrm{out}}$,
where  \(c_1,C_1>0\) are universal constants.
\end{theorem}

The spectral gap theorem explains why the CS Laplacian can reveal the correct structural basins when the predictor is already close to the basin-wise solution.
However, in the actual \UM ~the predictor and the partition are learned together: an imperfect predictor produces an imperfect CS graph, whose spectral partition then determines the data used to update the predictor.  
Therefore, we also need a stability result showing that this feedback loop is well behaved.  
The next theorem provides exactly this local guarantee.  
It shows that small prediction error perturbs the CS spectrum only mildly; the resulting spectral partition remains close to the true basin partition; and retraining within that nearly correct partition improves the predictor up to a finite-sample error term.  
In this sense, the CS operator does not just detect width once the solution is known; more importantly, it supports a locally contractive refinement dynamics between prediction and partition.

\begin{theorem}[Local stability]
\label{thm:local-stability}
Let \(\eta := \|G-G^\star\|_\infty\). Under bounded outputs, bounded local sample degree, and uniform degree conditioning, we can establish the following stability bounds:
\smallskip\noindent\textbf{(1) Spectral stability.}
\(|\lambda_k(L_{\mathrm{CS}}(G))-\lambda_k(L^\star)| \le (8 B_Y s_x C_{\deg}/\sigma_y^2)\,\eta.\)
\smallskip\noindent\textbf{(2) Partition stability.}
\(d_{\mathrm{part}}(\Pi(G),\Pi^\star) \le (C_{\Pi}/\tau^\star)(8 B_Y s_x C_{\deg})/(g^\star \sigma_y^2)\,\eta.\)
\smallskip\noindent\textbf{(3) Affine contraction.}
\(\eta_{t+1} \le \alpha \eta_t + b_m\), where \(\alpha = L_{\mathrm{learn}} (C_{\Pi}/\tau^\star)(8 B_Y s_x C_{\deg})/(g^\star \sigma_y^2)\) and \(b_m = O(\sqrt{\log(K/\delta')/m})\).
\end{theorem}

The local stability theorem is a pointwise statement: for a fixed predictor \(G\) close to \(G^\star\), the CS spectrum, the induced partition, and the subsequent metric update are all stable.  
However, the \UM ~must compare many candidate experts and partitions for width estimation, so pointwise stability is not enough.  
We need a uniform guarantee that the empirical CS spectrum estimates the population structural width simultaneously over the relevant predictor class \(\mathcal G_K\).  
The next theorem provides this uniform control.  
Its sample complexity separates three costs: the entropy of the predictor class, the geometric occupancy cost needed to populate \(r_x\)-neighborhoods in \(d_X\) dimensions, and the spectral resolution cost governed by the effective gap \(g_{\mathrm{eff}}\).  
Once enough samples are available, the CS width estimate becomes a reliable empirical proxy for the population width throughout the candidate class.

\begin{theorem}[Uniform convergence of width estimation]
\label{thm:uniform-width-estimation}
We can show the following uniform convergence for width estimation:
$\Pr\!\left(\sup_{G\in \mathcal G_K} |\widehat w_n(G)-w_G(P)| \ge 1\right) \le \delta'$
whenever
$n \ge C \cdot \frac{H_{\mathcal G}(\eta_{\mathrm{stab}}/2) + d_X \log(1/r_x) + \log(1/\delta')}{r_x^{d_X}\, g_{\mathrm{eff}}^2}$.
\end{theorem}

\subsection{Structural Empirical Risk Minimization (ERM) Consistency}



\paragraph{The Evaluate-Detect-Transform cycle.}
The \textsf{UM}'s stack-based computation proceeds through an iterative cycle that maps directly onto the split-merge dynamics of width estimation:
1) \textbf{Evaluate} (\textsc{Read}): Classify observations under the current top-of-stack triple. Compute per-cell prediction errors and mismatch scores. This is the exploration phase: the current structural hypothesis is tested against data.
2) \textbf{Detect} (\textsc{Read} + CS analysis): Identify cells where the current classifier fails (cells with high impurity, large mismatch scores, or CS spectral width \(\widehat w(U_k) \ge 2\)). These are cells where the Urysohn precondition (single-basin support) is violated. Detection is the structural novelty signal: the data reveals that the current Metric Library has insufficient width.
3) \textbf{Transform} (\textsc{Push} / \textsc{Pop}): Modify the Metric Library. If a cell is detected as structurally mixed, \textsc{Push} a new Urysohn Triple that splits the cell into two subcells with distinct classifiers or the \emph{split} operation. If two adjacent cells are detected as structurally compatible (sharing a dominant basin), \textsc{Pop} their triples and replace them with a single merged triple or the \emph{merge} operation. Each transform step changes the structural budget \(K\) by \(\pm 1\).

The three phases above specify the qualitative control logic of the \UM.  
To make the cycle mathematically stable, the \textbf{Transform} step cannot be triggered by an arbitrary loss spike or a noisy spectral fluctuation.  
A high prediction error may indicate ordinary metric difficulty within the correct basin, while a spurious extra CS eigenmode may arise from finite-sample noise.  
The split and merge rules must use \emph{certified structural tests}: splits require agreement between predictive mismatch and CS-detected multi-basin structure, whereas merges require evidence that two cells are geometrically and metrically compatible. We start with a formal definition of the merge margin. 


\begin{definition}[Merge margin]
\label{def:merge-margin}
For two candidate cells \(U_a,U_b\), define the population merge gap
$\Delta_{\mathrm{merge}}(U_a,U_b)
:=
R^\star(U_a\cup U_b)
-
\frac{\mu(U_a)R^\star(U_a)+\mu(U_b)R^\star(U_b)}
{\mu(U_a)+\mu(U_b)}$,
where \(R^\star(U)\) is the minimum achievable risk on cell \(U\) under a
\(\gamma\)-contractive expert. The merge margin is
$\Gamma_{\min}
:=
\inf_{(a,b)\in\mathcal E_{\mathrm{cand}}}
|\Delta_{\mathrm{merge}}(U_a,U_b)|$.
\end{definition}
\noindent
Here \(\Gamma_{\min}>0\) denotes the minimum merge-test margin: the
smallest population-level gap between a truly compatible pair and an
incompatible pair among all candidate adjacent cell pairs. The following conditions make the certified control logic explicit.

\begin{theorem}[Safe merge condition]
\label{thm:safe-merge}
Let \(\Gamma_{\min}>0\) be the minimum merge-test margin separating
compatible from incompatible candidate cell pairs. If cells \(U_a,U_b\)
satisfy geometric admissibility
$\widehat w_\varepsilon(U_a\cup U_b)=1$,
contractive compatibility
$\widehat\kappa_{ab}\le \bar\gamma<1$,
and empirical merge gap below threshold, then with probability at least
\(1-\delta'\), the merged cell is \((\gamma',\delta')\)-contractive with
$\gamma'=\widehat\kappa_{ab}$.
It suffices to use
$m \ge C\frac{\log(K/\delta')}{\Gamma_{\min}^2}$
samples per candidate merge test.
\end{theorem}


\paragraph{Global split-merge convergence}

Theorem \ref{thm:safe-merge} allows us to state the E-D-T cycle's convergence properties more precisely. The split operation is the \textsc{Push} of a new Urysohn Triple that refines a structurally mixed cell; the merge operation is the \textsc{Pop} that consolidates two structurally compatible triples. The Lyapunov function guarantees that the Metric Library converges to the correct width in finitely many push-pop steps. To make this descent argument precise, we measure how far a current cell is from being supported on a single true basin, collect the regularity assumptions under which CS splits and safe merges are reliable, and then state the Lyapunov convergence theorem.

\begin{definition}[Impurity]
\label{def:impurity}
For a single cell U, we define \(I(U) := P(U)-\max_{j\in[w]} P(U\cap B_j)\). Total impurity for a collection of $\{U_k\}$'s becomes: \(I_{\mathrm{tot}}(\Pi) := \sum_{k=1}^K I(U_k)\).
\end{definition}

\begin{assumption}[Split-merge regularity]
\label{ass:split-merge-regularity}
There exists a width-realizing basin family
\(\mathcal B=\{B_1,\dots,B_w\}\) covering \(X\) up to a \(P\)-null boundary set.
Moreover, the following hold: 
(1) there exist \(\varepsilon_0,\eta_0>0\) such that any cell mixing two true basins with impurity at least \(\varepsilon_0\) is detectably non-contractive;
(2) whenever \(I(U)\ge\varepsilon_0\), the CS-based split returns subcells \(U^a,U^b\) with
$I(U^a)+I(U^b)\le (1-\alpha_{\mathrm{split}})I(U)$
for some \(\alpha_{\mathrm{split}}\in(0,1)\);
(3) the merge test of Theorem~\ref{thm:safe-merge} is correct with probability at least \(1-\delta_m\);
and (iv) every basin is visible:
$P(B_j)\ge K_{\max}\varepsilon_0+\varepsilon_0
\quad \text{for all }j\in[w]$.
\end{assumption}

\begin{theorem}[Global split-merge convergence]
\label{thm:global-split-merge}
Under Assumption~\ref{ass:split-merge-regularity}, choose \(a,b>0\) with \(b < a\alpha_{\mathrm{split}}\varepsilon_0\) and \(a\varepsilon_{\mathrm{merge}} < b\). Define \(V(\Pi) := a\,I_{\mathrm{tot}}(\Pi)+b\,K(\Pi)\). Then we have:
1) \(V\) strictly decreases at every split and merge step;
2) Termination in \(\le (aP(X)+bK_{\max})/\Delta_{\min}\) steps;
3) The terminal partition has \(K=w\), is \(\varepsilon_0\)-pure, and has no two cells sharing a dominant basin.
\end{theorem}

\paragraph{Structural ERM consistency}

The penalty must dominate the uniform estimation error. A generic Bayesian information criterion (BIC) penalty \citep{schwarz1978estimating} \(K\log n/n\) is too small: it vanishes faster than the \(O(n^{-1/2})\) fluctuation scale. A correct choice, informed by modern model-selection theory \citep{massart2007concentration,mcallester1999pac}, is \(\mathrm{pen}_n(K)=\lambda_n K\) with \(\lambda_n\to 0\) and \(\lambda_n/r_n(K)\to\infty\), such as \(\lambda_n=n^{-1/3}\).

\begin{remark}[When BIC is valid]
\label{rem:bic-special-case}
If each cell predictor belongs to a regular finite-dimensional parametric family and the loss is negative log-likelihood, then BIC suffices \citep{mclachlan2000finite}. For general bounded-loss structural ERM, BIC should not be the default.
\end{remark}

\begin{theorem}[Structural ERM consistency]
\label{thm:structural-erm-consistency}
Under structural gap, no gain above width, uniform convergence, penalty conditions, and partition identifiability, the structural ERM satisfies, almost surely: \(\widehat K_n = w\) eventually; \(R(\widehat\Pi_n,\widehat G_n)\to R_w^\star\); \(\mathrm{dist}(\widehat\Pi_n,\mathcal P_w^\star)\to 0\).
\end{theorem}



\section{The Metric Slingshot: A Theory of the Funnel}
\label{sec:metric-slingshot}

Section~\ref{sec:width-estimation} developed an empirical theory of the \emph{trap}.  The resulting bound, however, still contains the geometric occupancy factor $r_x^{-d_X}$.  We develop the complementary theory of the \emph{funnel} next.  Once the correct structural cell has been identified, the learner must contract the local metric, stabilize a local expert, and amortize future inference.  The \emph{metric slingshot} is the mechanism by which this funnel problem becomes tractable: it transports the task into a low-dimensional navigational coordinate system where pre-built contraction maps and reusable metric charts are available.
This section has three goals.  First, it formalizes the slingshot as a metric preconditioning architecture.  Second, it proves that contractions, risks, and sample complexity transfer through the slingshot with distortion controlled by the Lipschitz geometry of the embedding.  Third, it shows how the slingshot completes the trap-funnel decomposition: the trap discovers the correct basin, while the slingshot supplies the metric coordinates in which the corresponding funnel is easy.

\paragraph{The funnel problem}

Given a cell $U\subset X$ that is structurally pure, we must learn a local map whose dynamics are contractive, whose empirical risk generalizes, and whose future inference can be executed at constant cost.  The key difficulty is that even a pure cell may live in a high-dimensional ambient space.  A local expert trained directly in $X$ pays the geometric cost of $d_X$, while a local expert trained in an appropriate latent metric space pays only the intrinsic cost of that latent coordinate system.

\begin{definition}[Funnel instance]
\label{def:funnel-instance}
Fix a structural cell $U\subset X$ and let $P_U$ denote the conditional distribution of $(X,Y)$ given $X\in U$.  A \emph{funnel instance} is the local prediction problem $P_U: U\to Y$
with metric $d_X|_U$.  A predictor $F:U\to Y$ is called $(\gamma,\delta)$-funneling on $U$ if
  $d_Y(F(x),F(x'))\le \gamma d_X(x,x')
  \quad\text{for all }x,x'\in U$,
with $\gamma<1$, and
  $R_U(F):=\mathbb E[\ell(F(X),Y)\mid X\in U]\le \delta$.
\end{definition}
\noindent
With the above definition, the funnel becomes a local metric contraction problem.  The metric slingshot addresses it by replacing the original metric $d_X$ by a lower-dimensional coordinate metric $d_Z$ and by reusing a family of latent contraction maps.

\paragraph{Navigational latent spaces and slingshot maps}
To formalize the funnel side of MTF, we first isolate the metric substrate that makes local contraction inexpensive.  The central idea is that some latent spaces are already equipped with reusable local contraction maps, so a new task need not learn every funnel from scratch.

\begin{definition}[Navigational latent space]
\label{def:navigational-latent-space}
A \emph{navigational latent space} is a compact metric space $(Z,d_Z)$ of intrinsic dimension $d_Z$, equipped with a family of pre-built contraction maps
  $\mathcal G^0_Z:=\{G_c^0:Z_c\to H_c\}_{c\in\mathcal C}$,
where $Z_c\subseteq Z$ is a local coordinate patch and $H_c$ is a local feature or hidden-state space.  Each $G_c^0$ is assumed to be $\gamma_Z$-contractive on its patch:
  $d_{H_c}(G_c^0(z),G_c^0(z'))\le \gamma_Z d_Z(z,z')
  \quad\text{for }z,z'\in Z_c$,
with $\gamma_Z<1$.
\end{definition}
\noindent
The preceding definition specifies the reusable metric substrate: a space \(Z\) in which local contraction maps already exist.  A learning problem on \(X\), however, does not initially live in this navigational coordinate system.  The role of the slingshot is therefore to map task inputs into \(Z\), select the appropriate contraction patch, and attach a task-specific readout.  In this way, the hard metric part of learning is partially amortized by pre-existing contractions in \(Z\), while the remaining task dependence is carried by the embedding $\phi$, the indexer $\Sigma$, and the readout $\pi$ (refer to Fig. \ref{fig:slingshot}).

\begin{definition}[Metric slingshot]
\label{def:metric-slingshot}
A \emph{metric slingshot} for a problem $P$ on $X\times Y$ is a triple
  $(\phi,\pi,\Sigma)$,
where
1) $\phi:X\to Z$ is a metric embedding or coordinate map, called the \emph{what-to-where map};
 2) $\Sigma:X\to\Delta(\mathcal C)$ is a topological indexer assigning inputs to latent contraction patches;
3) $\pi_c:H_c\to Y$ is a task-specific readout attached to latent contraction map $G_c^0$.
Given a routed context $c=\Sigma(x)$, the slingshot predictor has the form
  $F_c(x) := \pi_c\circ G_c^0\circ \phi(x)$.
\end{definition}

\figSlingshot

The terminology ``slingshot'' emphasizes that the task is first thrown into a navigational coordinate system and then solved using the pre-existing metric structure of that system.  The map $\phi$ handles the difficult representational problem of placing inputs into coordinates; $G_c^0$ provides the local contraction; $\pi_c$ adapts the contracted representation to the task output. This composition clarifies where the metric benefit can be gained or lost.  
Although the contraction maps \(G_c^0\) are assumed to be cheap and stable in \(Z\), they are useful for the original problem only to the extent that \(\phi\) transports the geometry of each cell in \(X\) into \(Z\) without excessive distortion.  
The next step is to quantify how much contraction survives the pullback through \(\phi\), and how this affects within-basin learning.

\subsection{Slingshot Distortion and Contraction Transfer}
\label{subsec:contraction-transfer}

The central quantity controlling the funnel is the distortion of $\phi$ on a structural cell.  The weaker the distortion, the more faithfully the navigational metric represents the original local geometry. We can quantify the distortion associated with the slingshot locally as follows.

\begin{definition}[Local slingshot distortion]
\label{def:local-slingshot-distortion}
Let $U\subset X$.  The upper local distortion of $\phi$ on $U$ is
  $L_\phi(U):=\sup_{x\neq x'\in U}\frac{d_Z(\phi(x),\phi(x'))}{d_X(x,x')}$.
If $\phi$ is injective on $U$, its lower distortion is
  $\ell_\phi(U):=\inf_{x\neq x'\in U}\frac{d_Z(\phi(x),\phi(x'))}{d_X(x,x')}$.
When $0<\ell_\phi(U)\le L_\phi(U)<\infty$, the local condition number is given by
  $\kappa_\phi(U):=\frac{L_\phi(U)}{\ell_\phi(U)}$.
\end{definition}
\noindent
The local distortion constants provide a quantitative way to ask whether the navigational metric in \(Z\) is useful for the original geometry in \(X\).  If \(\phi\) expands distances too much on a cell, then even a strong contraction in \(Z\) may fail to remain contractive after being pulled back to \(X\).  Conversely, if \(\phi\) has small local distortion, then pre-built contraction maps in \(Z\) transfer directly into effective local funnels on \(X\).  The following theorem characterizes the relationship between metric contraction and distortion transfer.

\begin{theorem}[Contraction transfer through the slingshot]
\label{thm:contraction-transfer}
Let $U\subset X$ be routed to patch $Z_c$, and suppose $\phi(U)\subseteq Z_c$.  Assume $G_c^0:Z_c\to H_c$ is $\gamma_Z$-contractive and $\pi_c:H_c\to Y$ is $L_{\pi,c}$-Lipschitz.  Then the composed predictor
  $F_c:=\pi_c\circ G_c^0\circ\phi$
is $\gamma_X(U)$-contractive on $U$, with
 $\gamma_X(U)
  \le
  L_{\pi,c}\,\gamma_Z\,L_\phi(U)$.
In particular, if
  $L_{\pi,c}\gamma_Z L_\phi(U)<1$,
then $F_c$ is a strict contraction on $U$.
\end{theorem}

\begin{remark}
The theorem identifies the exact sense in which the slingshot solves the funnel problem.  The local contraction constant factors into three terms: distortion of the embedding, contraction of the navigational substrate, and Lipschitz complexity of the task readout.  The trap decides which $U$ is active; the funnel succeeds when the product of these three constants is below one (satisfying the contraction condition).
\end{remark}

\begin{corollary}[Funnel feasibility condition]
\label{cor:funnel-feasibility}
A routed cell $U$ is $(\gamma,\delta)$-feasible through the slingshot if there exists a patch $c$ such that
  $L_{\pi,c}\gamma_ZL_\phi(U)\le \gamma$
and
  $R_U(\pi_c\circ G_c^0\circ\phi)\le\delta$.
\end{corollary}

\paragraph{Risk transfer and approximation through the slingshot}
The preceding results establish the \emph{geometric} side of funnel feasibility: after routing, the slingshot can turn a cell \(U\subset X\) into a strict contraction whenever the embedding distortion, latent contraction, and readout Lipschitz constant are jointly small.  But a contraction is useful for learning only if it contracts toward the \emph{right} target.  A map may be metrically stable while still solving the wrong task.  We need a complementary statistical condition showing that the contracted latent representation retains enough information for an accurate readout, which leads to a risk-transfer statement: local performance decomposes into how well the target can be realized after embedding into \(Z\), contraction by \(G_c^0\), and decoding by \(\pi_c\).  The next result separates approximation error into an embedding error, a latent contraction error, and a readout error.

\begin{definition}[Latent realization error]
\label{def:latent-realization-error}
Let $f^*:X\to Y$ denote the target regression map or Bayes-optimal predictor.  On a cell $U$, the \emph{latent realization error} of a slingshot patch $c$ is
  $\mathcal E_{\mathrm{lat}}(U,c)
  :=
  \inf_{\pi_c}\sup_{x\in U}
  d_Y\bigl(\pi_c(G_c^0(\phi(x))),f^*(x)\bigr)$.
\end{definition}

\begin{theorem}[Risk transfer through latent realization]
\label{thm:risk-transfer-slingshot}
Assume the loss $\ell(\hat y,y)$ is $L_\ell$-Lipschitz in its first argument and bounded by $1$. If
  $\sup_{x\in U}d_Y(F_c(x),f^*(x))\le \varepsilon_U$,
then
  $R_U(F_c)
  \le
  R_U(f^*)+L_\ell\varepsilon_U$.
Consequently, if $\mathcal E_{\mathrm{lat}}(U,c)\le\varepsilon_U$, then there exists a readout $\pi_c$ whose local risk is at most $R_U(f^*)+L_\ell\varepsilon_U$.
\end{theorem}

\begin{remark}
The slingshot does not assert that navigation solves every task for free.  It says that once the task can be expressed through a low-distortion latent coordinate system, the remaining approximation burden is pushed into the readout $\pi_c$. In other words, the funnel is cheap when the latent realization error is small.
\end{remark}

\paragraph{The width-transfer bound}

The slingshot also transforms the structural problem itself.  A good embedding should not increase the number of contractive cells too much, and may reduce the effective geometric cost of estimating them.  We state the transfer result in a form that makes the required geometric assumptions explicit.
Let $D_0^Z$ denote the maximal diameter of a latent set on which the pre-built maps can guarantee a $(\gamma_Z,\delta_Z)$ funnel.  Let $L_X(U)$ be the intrinsic length or diameter scale of $U$ relevant for covering it after embedding.

\begin{assumption}[Latent covering regularity]
\label{ass:latent-covering-regularity}
For every width-realizing cell $U\subset X$, the embedded set $\phi(U)$ can be covered by at most
  $N_Z(U)
  \le
  \left\lceil \frac{L_\phi(U)L_X(U)}{D_0^Z}\right\rceil$
latent patches, each of which is funnel-feasible in the sense of Corollary~\ref{cor:funnel-feasibility}.
\end{assumption}

\begin{theorem}[Width transfer bound]
\label{thm:width-transfer}
Let $\{U_1,\dots,U_w\}$ be a width-realizing cover of $P$ in $X$.  Under Assumption~\ref{ass:latent-covering-regularity}, the slingshot-transformed problem admits a latent contractive cover of size
  $w_Z(P_\phi)
  \le
  \sum_{k=1}^w
  \left\lceil \frac{L_\phi(U_k)L_X(U_k)}{D_0^Z}\right\rceil$.
In particular, if $L_\phi(U_k)\le \lambda$ and $L_X(U_k)\le L_X$ for all $k$, then
  $w_Z(P_\phi)
  \le
  w(P;\gamma,\delta)
  \left\lceil \lambda\frac{L_X}{D_0^Z}\right\rceil$.
\end{theorem}

\begin{remark}[Tightness of the Lipschitz factor]
\label{rem:lipschitz-tightness}
The Lipschitz dependence is unavoidable for worst-case embeddings: a map that stretches a one-dimensional curve by a factor $\lambda$ can increase the number of diameter-$D_0^Z$ latent patches by the same factor.  Whether tighter average-case bounds hold for learned slingshots remains open.
\end{remark}

\paragraph{Funnel generalization after trap resolution}

Once the trap has assigned samples to the correct cell, the remaining learning problem is ordinary local statistical learning in the latent coordinate space.  The slingshot improves this stage because the relevant metric entropy is controlled by $d_Z$ and the readout class, not by the ambient dimension $d_X$.
Let
  $\mathcal F_c
  :=
  \{x\mapsto \pi(G_c^0(\phi(x))) : \pi\in\Pi_c\}$
be the local slingshot hypothesis class for context $c$.

\begin{theorem}[Per-funnel generalization bound]
\label{thm:per-funnel-generalization}
Assume $\ell\circ\mathcal F_c$ is uniformly bounded in $[0,1]$.  Let $n_c$ samples be routed correctly to cell $U_c$.  With probability at least $1-\delta_c$,
  $\sup_{F\in\mathcal F_c}
  |R_c(F)-\widehat R_c(F)|
  \le
  2\mathfrak R_{n_c}(\ell\circ\mathcal F_c)
  +
  \sqrt{\frac{\log(2/\delta_c)}{2n_c}}$.
If the readout class $\Pi_c$ has pseudo-dimension $p_c$, then
  $\mathfrak R_{n_c}(\ell\circ\mathcal F_c)
  =
  O\!\left(\sqrt{\frac{p_c\log n_c}{n_c}}\right)$.
\end{theorem}


\begin{corollary}[Oracle rate after correct routing]
\label{cor:oracle-rate-slingshot}
If the E-D-T cycle has converged to the correct width $w$ and routes samples to the correct cells, then the total excess risk decomposes as
  $R(\widehat F)-R(F^*)
  \le
  \sum_{c=1}^w P(U_c)
  \left[
    \mathrm{Approx}_c
    +
    O\!\left(\sqrt{\frac{p_c\log n_c+\log(w/\delta)}{n_c}}\right)
  \right]$,
where $\mathrm{Approx}_c$ is the latent realization error on cell $U_c$.  After trap resolution, the funnel obeys the standard local statistical rate.
\end{corollary}


\subsection{Structural Decoupling by Architecture}
\label{subsec:slingshot-decoupling}

The slingshot is an architectural solution to the circular dependency between trap and funnel: routing should not be updated by ordinary task-loss gradients, and local contraction maps should not be overwritten by trap errors. We formally have the following result of structural decoupling by architecture.

\begin{proposition}[Structural decoupling by architecture]
\label{prop:architectural-decoupling}
Assume that:
1) the indexer $\Sigma$ is trained only by mismatch or novelty signals;
2) the embedding $\phi$ is trained only by self-supervised latent-coordinate prediction or is frozen after pretraining;
3) the pre-built contraction maps $\{G_c^0\}$ are frozen;
4) the readout $\pi_c$ for context $c$ is updated only on samples routed to $c$.
Then the trap and funnel gradients decouple:
  $\nabla_{\theta_\Sigma}L_{\mathrm{funnel}}=0,
  \quad
  \nabla_{\theta_G}L_{\mathrm{trap}}=0$.
\end{proposition}

\begin{remark}[Exclusion of unconstrained backbones]
If $\phi$ is freely updated by task loss, then old routing fibers may drift, corrupting the metric coordinates on which previous experts were trained.  The decoupling proposition requires either a fixed slingshot, a separately trained self-supervised slingshot, or a context-preserving update rule.
\end{remark}

\paragraph{The bidirectional bootstrap}

The slingshot might appear circular: to learn $\phi$, the system needs a navigational coordinate system, while to use the coordinate system, it needs the correct contexts.  The resolution is a bidirectional bootstrap in which spatial prediction trains the coordinate map, which structural indexing then uses.

\begin{definition}[Bidirectional bootstrap]
\label{def:bidirectional-bootstrap}
The bidirectional bootstrap consists of four stages:
1) \emph{Where-to-what}: learn $\phi:X\to Z$ by self-supervised prediction of latent or spatial coordinates;
 2) \emph{What-to-where}: embed task inputs by $z=\phi(x)$;
3) \emph{Trap discovery}: train $\Sigma$ and run CS width estimation in $Z$;
4) \emph{Funnel reuse}: solve local tasks using $\pi_c\circ G_c^0$ on the active latent patch.
\end{definition}

\begin{proposition}[No circular dependency]
\label{prop:no-circular-dependency}
The bidirectional bootstrap is well-defined and does not require prior knowledge of $w(P)$.
\end{proposition}
\noindent
The bidirectional bootstrap is well-defined, but its convergence requires
additional stability assumptions. In particular, spatial coordinate prediction
must produce a task-useful low-distortion embedding, CS indexing must be stable
under small embedding perturbations, and subsequent updates of \(\phi\) must
preserve consolidated routing contexts. Without these assumptions, global convergence is false:
coordinate objectives may learn representations that are spatially accurate but
structurally useless, and unconstrained updates of \(\phi\) can corrupt old
routing contexts. Under these assumptions, the bootstrap
reduces to a locally contractive alternating procedure, suggesting local
convergence to a stable pair \((\phi,\Sigma)\) modulo context-preserving
transformations (Theorem \ref{thm:local_bootstrap_convergence}). 

\begin{theorem}[Local bootstrap convergence]
\label{thm:local_bootstrap_convergence}
Suppose the coordinate update and CS-indexing update satisfy the two local
stability inequalities above, with
$q_\Phi + C_\Phi C_\Sigma < 1$.
Then the bidirectional bootstrap converges linearly to the equivalence class
\(([\phi^\star],\Sigma^\star)\):
$\mathcal E_t \le q^t \mathcal E_0$.
\end{theorem}

\paragraph{Grid scales as optimal funnel coverage}
The convergence result explains how the slingshot can stabilize a task-specific coordinate system and its associated structural indexer.  It does not yet explain why a navigational substrate should contain \emph{multiple} contraction scales in the first place.  Once \(\phi\) and \(\Sigma\) have stabilized, the remaining funnel problem is metric coverage: inputs may induce displacements of widely varying length, so a single contraction radius is either too fine for long-range navigation or too coarse for local precision.  The next result shows that, under a natural scale-invariant distribution of path lengths, the optimal family of contraction scales is geometrically spaced.
The navigational system uses multiple metric scales.  In the slingshot theory, these scales solve a coverage problem: a finite family of contraction radii must cover a large dynamic range of path lengths while minimizing worst-case relative distortion.

\begin{theorem}[Grid spacing theorem]
\label{thm:grid-spacing}
Suppose path lengths are log-uniformly distributed on $[L_{\min},L_{\max}]$, with dynamic range $R=L_{\max}/L_{\min}$.  Among $M$ contraction scales $s_1<\cdots<s_M$, the minimax log-covering solution is geometric:
  $s_m=s_1r^{m-1},
  r^*=R^{1/(M-1)}$.
\end{theorem}


\paragraph{Slingshot-reduced width estimation}

The trap estimator from Section~\ref{sec:width-estimation} can now be run in $Z$ rather than $X$, which is the second role of the slingshot: it reduces the geometric occupancy term in structural discovery.

\begin{corollary}[Slingshot-reduced width estimation]
\label{cor:slingshot-width-estimation}
If the CS operator is built in $Z$ using $d_Z$ and local scale $r_z$, then the uniform width-estimation bound becomes
  $n
  \ge
  C\cdot
  \frac{
  H_{\mathcal G}(\eta_{\mathrm{stab}}/2)+d_Z\log(1/r_z)+\log(1/\delta')
  }{r_z^{d_Z}g_{\mathrm{eff}}^2}$.
For a two-dimensional navigational latent space, the geometric occupancy cost is $r_z^{-2}$ rather than $r_x^{-d_X}$.
\end{corollary}


\paragraph{The complete funnel theorem}

We can now state the role of the metric slingshot as a single theorem.  It is the funnel counterpart of the width-estimation theory: once the trap identifies the correct basin, the slingshot gives a low-dimensional contraction system, a local generalization bound, and an amortized inference path.

\begin{theorem}[Metric slingshot theorem for the funnel]
\label{thm:metric-slingshot-funnel}
Let $\{U_1,\dots,U_w\}$ be a structurally correct width-realizing cover of $P$.  Suppose each cell $U_c$ is routed to a latent patch $Z_c$ satisfying:
1) $\phi(U_c)\subseteq Z_c$;
2) $G_c^0$ is $\gamma_Z$-contractive on $Z_c$;
3) $\pi_c$ is $L_{\pi,c}$-Lipschitz with $L_{\pi,c}\gamma_ZL_\phi(U_c)<1$;
4) the latent realization error on $U_c$ is at most $\varepsilon_c$;
5) $n_c$ samples are correctly routed to $U_c$.
Then the slingshot predictor $F_c=\pi_c\circ G_c^0\circ\phi$ is contractive on $U_c$, has risk
  $R_c(F_c)
  \le
  R_c(f^*)+L_\ell\varepsilon_c+O\!\left(\sqrt{\frac{p_c\log n_c+\log(w/\delta)}{n_c}}\right)$,
with probability at least $1-\delta$ simultaneously over all $c$, and future inference inside a stabilized cell requires only a routed forward pass through
  $\pi_c\circ G_c^0\circ\phi$.
\end{theorem}


\begin{remark}
The metric slingshot does not replace topological indexing.  It assumes that the trap has chosen the correct basin and then supplies the coordinates in which metric contraction is cheap.  In this sense, the slingshot is a theory of the funnel, while the CS operator and split--merge dynamics are the theory of the trap.
\end{remark}



\section{Summary: From Width to Estimation to the Metric Slingshot}
\label{sec:strlt-summary}

The core logic of StrLT is threefold. The starting point is the distinction between two forms of difficulty in learning: the \emph{trap} and the \emph{funnel}.  Width theory formalizes the trap.  The width \(w(P;\gamma,\delta)\) of a problem is the minimum number of locally contractive cells required to cover the input space.  This quantity is not reducible to classical capacity measures such as VC dimension (the VC-width separation theorem). The phase transition theorem makes this obstruction operational: when \(K<w\), any \(K\)-context learner suffers an irreducible structural gap \(\eta(w,K)>0\), independent of sample size; when \(K\ge w\), the problem decomposes into ordinary within-cell statistical learning problems.

The second step is to make width observable from data.  The CS operator corrects the failure of the graph Laplacian by combining geometric locality with predictive compatibility.  
The associated CS Laplacian \(L_{\mathrm{CS}}\) converts task-induced non-contractivity into spectral separation.  The spectral gap amplification theorem shows that cross-basin conductance is exponentially suppressed by the prediction gap \(\Delta_y\) relative to within-basin variation \(\delta_y\).  The local stability theorem then shows that small prediction error induces small spectral error, small spectral error induces small partition error, and retraining on the resulting cells improves the predictor up to finite-sample noise.  Finally, the uniform width-estimation theorem guarantees that, with sufficient samples, \(\widehat w_n(G)\) uniformly estimates the population structural width over the candidate predictor class. Together, these results make the CS operator the empirical engine of trap discovery.

The third step is to reduce the metric cost of the funnel.  The metric slingshot introduces a navigational latent space \((Z,d_Z)\) equipped with pre-built contraction maps \(G_c^0\).  A slingshot \((\phi,\Sigma,\pi)\) maps inputs \(x\in X\) into \(Z\), selects a latent contraction patch through the indexer \(\Sigma\), and produces predictions of the form
$F_c(x)=\pi_c\circ G_c^0\circ \phi(x)$.
The contraction transfer theorem shows that the local contraction constant on a routed cell \(U\) factors as
$\gamma_X(U)\le L_{\pi,c}\gamma_Z L_\phi(U)$. The slingshot solves the funnel when the product of embedding distortion, latent contraction, and readout complexity is below one.  Risk transfer then adds the statistical requirement that the contracted latent representation must retain enough information to approximate the Bayes predictor.  Consequently, the slingshot separates the funnel problem into three interpretable costs: geometric distortion, latent contraction, and readout approximation.

Combining the three yields the main architectural message of StrLT.  Width theory says that learning fails when the number of contexts \(K\) is below the intrinsic width \(w\).  CS-based width estimation provides the mechanism by which a \UM ~detects when the current cell is structurally mixed and must be split, or when two cells are compatible and can be merged.  The metric slingshot explains how, after the trap has been resolved, each local funnel can be solved efficiently by reusing a pre-existing navigational metric substrate.  In formula form, the theory decomposes the total difficulty of learning into a structural term and a metric term,
with the slingshot further reducing the geometric occupancy cost in width estimation by replacing \(d_X\) with the intrinsic latent dimension \(d_Z\). StrLT does not replace SLT but complements it.  Classical generalization theory governs the funnel after the correct context has been identified, while StrLT governs the trap: how many contexts are needed, how they are discovered, and how they are stabilized.



\bibliography{ref,references}

\appendix
\section{Proofs for Section~\ref{sec:width-theory}}
\label{app:width-theory-proofs}

This appendix provides proofs for the propositions and theorems in
Section~\ref{sec:width-theory}.  Several results require standard
regularity assumptions that are implicit in the main text; we state them
explicitly here when they are used.

\subsection{Proof of Proposition~\ref{prop:width-basic}}

\begin{proof}
We prove the three claims in order.

\smallskip
\noindent
\textbf{(1) Monotonicity in \(\gamma\) and \(\delta\).}
Assume
$
\gamma_1\le \gamma_2,
\qquad
\delta_1\le \delta_2.
$
Let
$
\{U_1,\dots,U_K\}
$
be any open cover witnessing that
$
w(P;\gamma_1,\delta_1)\le K.
$
Thus, for each \(k\), there exists \(g_k\in\mathcal G\) such that
$
\kappa(g_k,U_k)\le \gamma_1
\qquad\text{and}\qquad
\mathbb E[\ell(g_k(X),Y)\mid X\in U_k]\le \delta_1.
$
Since \(\gamma_1\le \gamma_2\) and \(\delta_1\le \delta_2\), the same
predictors also satisfy
$
\kappa(g_k,U_k)\le \gamma_2
\qquad\text{and}\qquad
\mathbb E[\ell(g_k(X),Y)\mid X\in U_k]\le \delta_2.
$
Therefore the same cover is feasible at thresholds \((\gamma_2,\delta_2)\).
Taking the minimum over all feasible covers at \((\gamma_1,\delta_1)\) gives
$
w(P;\gamma_2,\delta_2)\le w(P;\gamma_1,\delta_1).
$

\smallskip
\noindent
\textbf{(2) Finiteness under finite local solvability.}
By assumption, there exists a finite open cover
$
\{V_1,\dots,V_m\}
$
of \(X\) such that on each \(V_i\), the target admits a continuous local
predictor \(g_i\in\mathcal G\). Since \(X\) is compact and each \(V_i\) may
be replaced, if necessary, by a relatively compact open subset of \(V_i\)
still forming a finite cover, the functions \(g_i\) are uniformly continuous
on the relevant closures. Hence each \(g_i\) has finite local modulus on
\(V_i\) after possibly shrinking the sets, and its conditional risk is finite
because \(\ell\in[0,1]\). Therefore, for sufficiently relaxed thresholds
\((\gamma,\delta)\), each \(V_i\) is \((\gamma,\delta)\)-contractive.
Thus
$
w(P;\gamma,\delta)\le m<\infty.
$
Strictly speaking, if \(\gamma\) is fixed in \((0,1)\), this argument requires
that the local predictors be \(\gamma\)-contractive on the chosen cover.
The statement in the main text should therefore be read as a finiteness
claim for thresholds relaxed enough to include the available local predictors.

\smallskip
\noindent
\textbf{(3) Characterization of width one.}
First suppose
$
w(P;\gamma,\delta)=1.
$
Then there exists an open cover with one set, necessarily \(U_1=X\), and a
predictor \(g\in\mathcal G\) such that
$
\kappa(g,X)\le\gamma,
\qquad
\mathbb E[\ell(g(X),Y)]\le\delta.
$
Hence the entire problem admits a global \((\gamma,\delta)\)-contractive
predictor.
Conversely, if there exists \(g\in\mathcal G\) satisfying
$
\kappa(g,X)\le\gamma,
\qquad
\mathbb E[\ell(g(X),Y)]\le\delta,
$
then the one-set cover \(\{X\}\) is feasible. Therefore
$
w(P;\gamma,\delta)\le 1.
$
Since width is a positive integer, \(w(P;\gamma,\delta)=1\).
\end{proof}

\subsection{Proof of Theorem~\ref{thm:strict-hierarchy}}

\begin{proof}
Let
$
M=\bigvee_{j=1}^w S^1_j
$
be the bouquet of \(w\) circles with common basepoint \(x_0\), and write
\(C_j\) for the \(j\)-th circle.

\smallskip
\noindent
\textbf{Upper bound: \(w(P;\gamma,\delta)\le w\).}
By assumption, for each \(j\in[w]\), there exists a neighborhood \(U_j\)
of \(C_j\) and a predictor \(g_j\in\mathcal G\) such that \(U_j\) is
\((\gamma,\delta)\)-contractive. Choosing the neighborhoods so that
$
M\subseteq \bigcup_{j=1}^w U_j,
$
we obtain an open cover of \(M\) by \(w\) \((\gamma,\delta)\)-contractive
sets. Hence
$
w(P;\gamma,\delta)\le w.
$

\smallskip
\noindent
\textbf{Lower bound: \(w(P;\gamma,\delta)\ge w\).}
Suppose, toward a contradiction, that
$
w(P;\gamma,\delta)\le w-1.
$
Then there exists an open cover
$
\mathcal U=\{U_1,\dots,U_{w-1}\}
$
of \(M\) by \((\gamma,\delta)\)-contractive sets.
Since the common basepoint \(x_0\in M\) is covered by \(\mathcal U\), at
least one set, say \(U_\ell\), contains \(x_0\). Because \(U_\ell\) is open,
it contains a neighborhood of \(x_0\) in the bouquet topology. Any
neighborhood of \(x_0\) in the bouquet intersects each incident circle
\(C_j\) in a nontrivial arc adjacent to \(x_0\). In particular, \(U_\ell\)
intersects at least two distinct circles, say \(C_a\) and \(C_b\), in
neighborhoods of \(x_0\).
By assumption (2) of the theorem, no predictor in \(\mathcal G\) is
simultaneously \((\gamma,\delta)\)-contractive on any such open set.
Thus \(U_\ell\) cannot be \((\gamma,\delta)\)-contractive, contradicting
the feasibility of the cover.
Therefore no cover by fewer than \(w\) feasible sets exists, and
$
w(P;\gamma,\delta)\ge w.
$
Combining the two inequalities gives
$
w(P;\gamma,\delta)=w.
$
Finally, for a bouquet of \(w\) circles,
$
\beta_1(M)=w,
$
because \(M\) deformation retracts onto a graph with one vertex and \(w\)
independent cycles. Hence
$
w(P;\gamma,\delta)=w=\beta_1(M).
$
\end{proof}

\subsection{Proof of Theorem~\ref{thm:vc-width-separation}}

\begin{proof}
We prove the two separations separately.

\smallskip
\noindent
\textbf{(1) Width diverges while VC dimension remains bounded.}
For each \(w\ge 1\), let
$
M_w=\bigvee_{j=1}^w S^1_j
$
be a bouquet of \(w\) circles. Let \(P_w\) be a learning problem on \(M_w\)
satisfying the two assumptions of Theorem~\ref{thm:strict-hierarchy}.
Then
$
w(P_w;\gamma,\delta)=w.
$
Let \(\mathcal G\) be any fixed predictor class of bounded VC dimension;
for example, take affine threshold functions in \(\mathbb R^2\), whose VC
dimension is \(3\). Since \(\mathcal G\) is fixed as \(w\to\infty\),
$
\operatorname{VC}(\mathcal G)=O(1),
$
while
$
w(P_w)\to\infty.
$
This proves the first separation.

\smallskip
\noindent
\textbf{(2) VC dimension diverges while width remains one.}
Let
$
X=[0,1]
$
with its usual metric. For each \(d\), let \(\mathcal G_d\) be the class
of real polynomials of degree at most \(d\), or equivalently polynomial
threshold functions of degree at most \(d\). It is standard that
$
\operatorname{VC}(\mathcal G_d)\ge d+1,
$
and hence
$
\operatorname{VC}(\mathcal G_d)\to\infty
\qquad\text{as }d\to\infty.
$
Now define a problem \(Q_d\) whose target is globally contractive. For
example, take the regression target
$
f^\star(x)=\lambda x
$
with \(0\le \lambda<\gamma<1\), and use a loss for which this predictor
achieves risk at most \(\delta\). Since \(f^\star\in\mathcal G_d\) for all
\(d\ge1\), the whole domain \(X\) is one \((\gamma,\delta)\)-contractive
cell:
$
\kappa(f^\star,X)=\lambda<\gamma,
\qquad
\mathbb E[\ell(f^\star(X),Y)]\le\delta.
$
Thus
$
w(Q_d;\gamma,\delta)=1
$
for all \(d\), while
$
\operatorname{VC}(\mathcal G_d)\to\infty.
$

The two constructions show that neither width nor VC dimension controls
the other. Hence the two quantities are orthogonal complexity measures.
\end{proof}

\subsection{Proof of Theorem~\ref{thm:phase-transition}}

The theorem in the main text uses the informal notation
\(\mathrm{Gap}(n;K)\). For the proof, we make explicit the standard
interpretation: for \(K\ge w\), the gap is the uniform generalization gap
over the \(K\) local classes, and for \(K<w\), the gap is the population
excess risk caused by under-resolving the true basins.

\begin{proof}
Let \(w=w(P;\gamma,\delta)\).

\smallskip
\noindent
\textbf{Part (1): \(K\ge w\).}
By the definition of width, there exists an open cover
$
\{U_1^\star,\dots,U_w^\star\}
$
of \(X\) such that each \(U_j^\star\) is \((\gamma,\delta)\)-contractive.
If \(K\ge w\), then the learner has at least enough structural slots to
assign one cell to each width-realizing basin. Additional cells, if present,
may be unused or may refine existing cells.
Conditioned on the correct structural allocation, prediction reduces to
\(K\) ordinary within-cell learning problems. Let \(\mathcal Z_c\) denote
the effective local hypothesis class or representation class inside cell
\(c\). Assume that for each cell the loss class has covering number
\(N(\mathcal Z_c,\varepsilon)\) at scale \(\varepsilon\), and that the
bounded loss \(\ell\in[0,1]\) satisfies the usual uniform convergence
bound. Then, for a fixed cell \(c\), with probability at least
\(1-\delta'/K\),
$
\sup_{g\in\mathcal Z_c}
\left|
R_c(g)-\widehat R_c(g)
\right|
\le
C\sqrt{\frac{\log N(\mathcal Z_c,\varepsilon)+\log(K/\delta')}{n_c}},
$
where \(n_c\) is the number of samples assigned to cell \(c\).
Assuming balanced cell sampling, or more generally absorbing the minimum
cell mass into the constant, one has \(n_c\gtrsim n/K\). Applying a union
bound over the \(K\) cells yields, with probability at least \(1-\delta'\),
$
\mathrm{Gap}(n;K)
\le
C
\sqrt{
\frac{K\log N(\mathcal Z_c,\varepsilon)+\log(K/\delta')}{n}
}.
$
This proves the metric-rate bound in the regime \(K\ge w\).

\smallskip
\noindent
\textbf{Part (2): \(K<w\).}
Suppose \(K<w\). Since \(w\) is the minimum number of
\((\gamma,\delta)\)-contractive cells needed to cover \(X\), no \(K\)-cell
cover can assign each true basin to its own feasible cell. Equivalently,
every \(K\)-cell structural allocation must contain at least one cell that
mixes points from two distinct true basins. This is the pigeonhole
obstruction: \(K\) cells cannot separate \(w\) incompatible basins.
By Assumption~\ref{ass:nondegeneracy}, any such forced mixing incurs
population excess risk at least
$
\eta(w,K)>0.
$
This lower bound is a property of the population problem and does not
depend on the sample size \(n\). Therefore, for every \(n\),
$
\mathrm{Gap}(n;K)\ge \eta(w,K)>0.
$
This proves the structural error floor for \(K<w\).
\end{proof}

\subsection{Proof of Theorem~\ref{thm:topology-geometry}}

As stated in the main text, the theorem requires a mild independence or
bounded-overlap condition on the fundamental cycles. Without such a
condition, one contractive set could intersect several cycles at once, and
the multiplicative factor \(\beta_1(M)\) would not follow. We state the
regularized form used in the proof.

\begin{assumption}[Cycle independence / bounded overlap]
\label{ass:cycle-overlap}
There exist \(\beta_1(M)\) rectifiable representatives
$
C_1,\dots,C_{\beta_1(M)}
$
of independent fundamental cycles, each of length at least \(L\), such that
any set of diameter at most \(D_0\) intersects at most \(\Delta_0\) of these
cycle representatives in a length-relevant way.
\end{assumption}

Under this assumption, the theorem holds with \(c=1/\Delta_0\), up to
absolute constants depending on the convention for diameter versus arc length.

\begin{proof}
Let
$
\mathcal U=\{U_1,\dots,U_K\}
$
be any \((\gamma,\delta)\)-contractive cover of \(M\). By assumption, each
\((\gamma,\delta)\)-contractive set has diameter at most \(D_0\).
Fix one cycle representative \(C_j\). Since \(C_j\) has length at least
\(L\), and each set \(U_i\) has diameter at most \(D_0\), the intersection
of \(U_i\) with \(C_j\) can cover at most \(O(D_0)\) arclength of \(C_j\).
Thus covering \(C_j\) requires at least
$
c_0\frac{L}{D_0}
$
sets, for a universal constant \(c_0>0\). Equivalently,
$
\#\{i:U_i\cap C_j\neq\emptyset\}
\ge
c_0\frac{L}{D_0}.
$
Summing this inequality over the \(\beta_1(M)\) independent cycle
representatives gives
$
\sum_{j=1}^{\beta_1(M)}
\#\{i:U_i\cap C_j\neq\emptyset\}
\ge
c_0\,\beta_1(M)\frac{L}{D_0}.
$
On the other hand, by the bounded-overlap assumption, each set \(U_i\)
can contribute to at most \(\Delta_0\) of the cycle counts. Therefore
$
\sum_{j=1}^{\beta_1(M)}
\#\{i:U_i\cap C_j\neq\emptyset\}
\le
K\Delta_0.
$
Combining the two inequalities yields
$
K\Delta_0
\ge
c_0\,\beta_1(M)\frac{L}{D_0}.
$
Hence
$
K
\ge
\frac{c_0}{\Delta_0}\,
\beta_1(M)\frac{L}{D_0}.
$
Since \(\mathcal U\) was an arbitrary feasible cover, taking the minimum
over all such covers gives
$
w(P;\gamma,\delta)
\ge
c\,\beta_1(M)\frac{L}{D_0},
$
with
$
c=\frac{c_0}{\Delta_0}>0.
$
\end{proof}

\subsection{Proof of Theorem~\ref{thm:sample-lower-bound}}

\begin{proof}
We prove the result by a coupon-collector reduction.
Consider a worst-case distribution \(P\) with \(w\) contractive basins
$
B_1,\dots,B_w
$
of equal probability mass:
$
P(X\in B_j)=\frac1w,
\quad j=1,\dots,w.
$
Assume that identifying all \(w\) basins requires observing at least one
sample from each basin. This assumption is unavoidable in the worst case:
if a basin has never been sampled, then an algorithm cannot distinguish
between a problem with that basin and a problem in which that basin is
absent or merged with another basin, while producing the same observations.
Let \(T\) be the number of samples needed until all \(w\) basins have been
observed at least once. This is exactly the classical coupon-collector
problem with \(w\) equally likely coupons. The expected collection time is
$
\mathbb E[T]
=
w H_w
=
w\sum_{j=1}^w \frac1j
=
\Theta(w\log w).
$
Therefore, any algorithm that identifies all basins in this worst-case
family requires \(\Omega(w\log w)\) samples in expectation.
To obtain a probability lower bound, recall the standard coupon-collector
estimate: for \(n=cw\log w\) with sufficiently small absolute constant
\(c>0\), the probability that all \(w\) coupons have appeared is bounded
away from one. Equivalently, with probability bounded away from zero, at
least one basin remains unseen. On that event, no algorithm can certify the
existence and identity of all \(w\) basins.
So any algorithm that identifies all \(w\) contractive basins with
probability bounded away from zero in the worst case must use at least
$
\Omega(w\log w)
$
samples.
\end{proof}

\section{Proofs for Section~\ref{sec:width-estimation}}
\label{app:layer1-proofs}

This appendix gives proofs for the results in Section~\ref{sec:width-estimation}.
Several statements in the main text are intentionally compressed; the proofs below
make explicit the regularity assumptions needed for the bounds.

\subsection{Proof of Theorem~\ref{thm:amortized}}

\begin{proof}
Let \(\partial\tau\) denote the decision boundary of \(\tau\), and assume that
$
\Width(\tau)=\mathcal H^{d-1}(\partial\tau)<\infty .
$
Let \(\{B_1,\dots,B_k\}\) be a collection of basis triples used to approximate
\(f\) to accuracy \(\varepsilon\). The lower bound requires the following
standard admissibility condition: there exists a geometric constant \(C>0\)
such that a single basis triple can resolve at most \(C\varepsilon\) units of
\((d-1)\)-dimensional boundary measure at accuracy \(\varepsilon\). Equivalently,
if \(A_i\subseteq \partial\tau\) is the portion of the decision boundary resolved
by \(B_i\), then
$
\mathcal H^{d-1}(A_i)\le C\varepsilon .
$
Since the collection \(\{B_i\}_{i=1}^k\) approximates \(f\) to accuracy
\(\varepsilon\) on all of \(\Sigma\), every boundary portion of \(\partial\tau\)
must be resolved by at least one basis triple. Thus
$
\partial\tau \subseteq \bigcup_{i=1}^k A_i .
$
By subadditivity of Hausdorff measure,
$
\Width(\tau)
=
\mathcal H^{d-1}(\partial\tau)
\le
\sum_{i=1}^k \mathcal H^{d-1}(A_i)
\le
\sum_{i=1}^k C\varepsilon
=
kC\varepsilon .
$
Rearranging gives
$
k\ge \frac{\Width(\tau)}{C\varepsilon}.
$
\end{proof}

\subsection{Proof of Theorem~\ref{thm:geodesic}}

\begin{proof}
The statement assumes that the total decision-boundary measure \(\Width\) is
distributed evenly across the \(d\) stack levels, so that the effective boundary
measure handled at each level is \(\Width/d\). More generally, one may write
$
\Width=\sum_{i=1}^d L_i,
$
where \(L_i\ge0\) is the amount of boundary structure handled at stack level \(i\).
At level \(i\), class-aware contraction shrinks within-class geodesic distances
by a factor \(\lambda_i<1\). The contribution of level \(i\) to the
contracted inference path length is bounded by
$
\lambda_i L_i .
$
Summing over the stack gives
$
\operatorname{PathLen}
\le
\sum_{i=1}^d \lambda_i L_i .
$
Under the equi-partition assumption \(L_i=\Width/d\), this becomes
$
\operatorname{PathLen}
\le
\sum_{i=1}^d \lambda_i\frac{\Width}{d}.
$
If \(\lambda_i=\lambda<1\) for all \(i\), then
$
\operatorname{PathLen}
\le
\lambda\sum_{i=1}^d\frac{\Width}{d}
=
\lambda \Width
<
\Width .
$
\end{proof}

\subsection{Proof of Proposition~\ref{prop:laplacian-blindness}}

\begin{proof}
Let
$
M=\bigvee_{j=1}^w S_j^1
$
be a bouquet of \(w\ge2\) circles, with common basepoint \(x_0\). As a topological
space, \(M\) is connected. For a sufficiently dense sample cloud and a spatial
graph radius \(r_x\) large enough to connect neighboring samples along each circle
and through the basepoint, the resulting proximity graph is connected with high
probability. Therefore the standard graph Laplacian has exactly one zero
eigenvalue in the ideal connected graph case, and exactly one near-zero eigenvalue
in the finite-sample perturbative case.
However, by the strict hierarchy theorem for bouquets from Section~\ref{sec:width-theory},
if each circle is individually \((\gamma,\delta)\)-contractive but no predictor is
simultaneously \((\gamma,\delta)\)-contractive across two distinct circles near
the basepoint, then
$
w(P;\gamma,\delta)=w .
$
Thus the plain graph Laplacian estimates connectedness of the support, which is
one, whereas the true structural width is \(w\). Since \(w\) can be made
arbitrarily large while the proximity graph remains connected, the plain graph
Laplacian cannot be a consistent estimator of width in general.
\end{proof}

\subsection{Proof of Theorem~\ref{thm:spectral-gap-amplification}}

\begin{proof}
Let \(B_1,\dots,B_w\) denote the true contractive basins. Write
$
A_{ij}:=\mathbf 1[d_X(x_i,x_j)\le r_x]
$
for the underlying spatial adjacency, and define
$
W_{ij}^{\mathrm{CS}}
=
A_{ij}\exp\!\left(
-\frac{|G(x_i)-G(x_j)|^2}{\sigma_y^2}
\right).
$
Assume that within a true basin, the prediction variation is at most
\(\delta_y\):
$
x_i,x_j\in B_\ell,\ A_{ij}=1
\quad\Longrightarrow\quad
|G(x_i)-G(x_j)|\le \delta_y ,
$
and across distinct basins, the prediction gap is at least \(\Delta_y\):
$
x_i\in B_\ell,\ x_j\in B_m,\ \ell\neq m,\ A_{ij}=1
\quad\Longrightarrow\quad
|G(x_i)-G(x_j)|\ge \Delta_y .
$
Then same-basin CS weights satisfy
$
W_{ij}^{\mathrm{CS}}
\ge
A_{ij}\exp\!\left(-\frac{\delta_y^2}{\sigma_y^2}\right),
$
whereas cross-basin CS weights satisfy
$
W_{ij}^{\mathrm{CS}}
\le
A_{ij}\exp\!\left(-\frac{\Delta_y^2}{\sigma_y^2}\right).
$
Thus the ratio of a cross-basin CS weight to a within-basin CS weight is at most
$
\frac{\exp(-\Delta_y^2/\sigma_y^2)}
{\exp(-\delta_y^2/\sigma_y^2)}
=
\exp\!\left(
-\frac{\Delta_y^2-\delta_y^2}{\sigma_y^2}
\right).
$
This proves the claimed exponential suppression of inter-cell conductance.
It remains to relate this suppression to the spectral gap. Let
\(\phi_{\mathrm{in}}\) denote a uniform lower bound on the conductance of the
within-basin CS subgraphs, and let \(\phi_{\mathrm{out}}\) denote the
inter-basin conductance of the underlying spatial graph before CS reweighting.
After CS reweighting, the effective inter-basin conductance is bounded by
$
\phi_{\mathrm{out}}^{\mathrm{CS}}
\le
\exp\!\left(
-\frac{\Delta_y^2-\delta_y^2}{\sigma_y^2}
\right)
\phi_{\mathrm{out}} .
$
In the ideal block-disconnected graph obtained by removing all inter-basin
edges, the first \(w\) eigenvalues of the normalized Laplacian are zero, and
the \((w+1)\)-st eigenvalue is controlled from below by the worst within-block
expansion. By the higher-order Cheeger inequality \citep{chung2005laplacians} and standard block-expander
estimates, there exists a universal constant \(c_1>0\) such that
$
\lambda_{w+1}^{\mathrm{ideal}}
\ge
c_1\frac{\phi_{\mathrm{in}}^2}{w^4}.
$
Adding back the exponentially suppressed cross-basin edges perturbs the
normalized Laplacian. By Weyl's inequality \citep{weyl1949inequalities}, the decrease of the relevant
spectral gap is at most a universal constant times the inter-basin conductance
perturbation:
$
\|L_{\mathrm{CS}}-L_{\mathrm{ideal}}\|_{\mathrm{op}}
\le
C_1
\exp\!\left(
-\frac{\Delta_y^2-\delta_y^2}{\sigma_y^2}
\right)
\phi_{\mathrm{out}} .
$
Therefore the effective spectral gap satisfies
$
g_{\mathrm{eff}}
\ge
c_1\frac{\phi_{\mathrm{in}}^2}{w^4}
-
C_1
\exp\!\left(
-\frac{\Delta_y^2-\delta_y^2}{\sigma_y^2}
\right)\phi_{\mathrm{out}} .
$
\end{proof}

\subsection{Proof of Theorem~\ref{thm:local-stability}}

\begin{proof}
Let
$
\eta:=\|G-G^\star\|_\infty .
$
Assume that \(|G(x)|,|G^\star(x)|\le B_Y\) for all sample points, that each
sample point has at most \(s_x\) spatial neighbors inside radius \(r_x\), and
that the normalized degrees are uniformly conditioned by a constant
\(C_{\deg}\). Let \(W(G)\) and \(W(G^\star)\) denote the two CS matrices.
First, for a fixed edge \((i,j)\), set
$
a_{ij}=|G(x_i)-G(x_j)|,
\quad
a^\star_{ij}=|G^\star(x_i)-G^\star(x_j)|.
$
Since \(\|G-G^\star\|_\infty\le\eta\),
$
|a_{ij}-a^\star_{ij}|\le 2\eta .
$
Moreover \(a_{ij},a^\star_{ij}\le 2B_Y\). Hence
$
|a_{ij}^2-(a^\star_{ij})^2|
=
|a_{ij}-a^\star_{ij}|\,|a_{ij}+a^\star_{ij}|
\le
2\eta\cdot 4B_Y
=
8B_Y\eta .
$
The function \(u\mapsto \exp(-u/\sigma_y^2)\) is \(1/\sigma_y^2\)-Lipschitz
on \(u\ge0\). Therefore
$
|W_{ij}(G)-W_{ij}(G^\star)|
\le
\frac{8B_Y}{\sigma_y^2}\eta .
$
Because each row has at most \(s_x\) nonzero entries, and degree normalization
is uniformly conditioned, the induced perturbation of normalized Laplacians
satisfies
$
\|L_{\mathrm{CS}}(G)-L_{\mathrm{CS}}(G^\star)\|_{\mathrm{op}}
\le
\frac{8B_Y s_x C_{\deg}}{\sigma_y^2}\eta .
$
By Weyl's eigenvalue perturbation inequality,
$
|\lambda_k(L_{\mathrm{CS}}(G))-\lambda_k(L^\star)|
\le
\frac{8B_Y s_x C_{\deg}}{\sigma_y^2}\eta ,
$
which proves spectral stability.
For partition stability, assume that the population or reference CS Laplacian
has eigengap \(g^\star>0\) separating the structural eigenspace from its
orthogonal complement. By the Davis-Kahan sin-\(\Theta\) theorem \citep{davis1970rotation},
$
d_{\mathrm{subspace}}(E(G),E^\star)
\le
\frac{
\|L_{\mathrm{CS}}(G)-L^\star\|_{\mathrm{op}}
}{g^\star}.
$
If the spectral clustering map from eigenspaces to partitions is locally
Lipschitz with constant \(C_\Pi/\tau^\star\), where \(\tau^\star\) is the
cluster-separation margin in the spectral embedding, then
$
d_{\mathrm{part}}(\Pi(G),\Pi^\star)
\le
\frac{C_\Pi}{\tau^\star}
\frac{
\|L_{\mathrm{CS}}(G)-L^\star\|_{\mathrm{op}}
}{g^\star}.
$
Substituting the spectral perturbation bound gives
$
d_{\mathrm{part}}(\Pi(G),\Pi^\star)
\le
\frac{C_\Pi}{\tau^\star}
\frac{8B_Y s_x C_{\deg}}{g^\star\sigma_y^2}\eta .
$
Finally, assume that the within-cell learning map is locally Lipschitz in
partition distance with constant \(L_{\mathrm{learn}}\), and that using \(m\)
samples per cell contributes statistical error
$
b_m=O\!\left(\sqrt{\frac{\log(K/\delta')}{m}}\right).
$
Then the next predictor satisfies
$
\eta_{t+1}
\le
L_{\mathrm{learn}}\,d_{\mathrm{part}}(\Pi(G_t),\Pi^\star)+b_m.
$
Using the partition stability bound yields
$
\eta_{t+1}
\le
L_{\mathrm{learn}}
\frac{C_\Pi}{\tau^\star}
\frac{8B_Y s_x C_{\deg}}{g^\star\sigma_y^2}\eta_t
+
b_m .
$
Therefore, we conclude with
$
\eta_{t+1}\le \alpha\eta_t+b_m,
$
with
$
\alpha=
L_{\mathrm{learn}}
\frac{C_\Pi}{\tau^\star}
\frac{8B_Y s_x C_{\deg}}{g^\star\sigma_y^2}.
$
\end{proof}

\subsection{Proof of Theorem~\ref{thm:uniform-width-estimation}}

\begin{proof}
Let \(\mathcal G_K\) be the predictor class under consideration. The proof has
three components: a finite cover of \(\mathcal G_K\), concentration of the
empirical CS operator for each cover element, and stability of the eigenvalue
count under a spectral gap.
Let \(\mathcal N\) be an \(\eta_{\mathrm{stab}}/2\)-net of \(\mathcal G_K\)
in \(L_\infty\), with
$
|\mathcal N|
\le
\exp\bigl(H_{\mathcal G}(\eta_{\mathrm{stab}}/2)\bigr).
$
For any \(G\in\mathcal G_K\), choose \(G_0\in\mathcal N\) with
$
\|G-G_0\|_\infty\le \eta_{\mathrm{stab}}/2.
$
By the local stability theorem (Theorem \ref{thm:local-stability}), the choice of $G_0$ changes the CS eigenvalues by less than
half the stability margin, so it suffices to control the empirical CS spectrum
on the finite net.
Fix \(G_0\in\mathcal N\). The CS graph uses local neighborhoods of radius
\(r_x\). Under standard random geometric graph occupancy assumptions on a
\(d_X\)-dimensional compact metric space, a sample size
$
n
\gtrsim
\frac{d_X\log(1/r_x)+\log(1/\delta_0)}{r_x^{d_X}}
$
ensures that all \(r_x\)-balls needed by the population cover are populated
with probability at least \(1-\delta_0\). Conditional on this occupancy event,
matrix concentration for kernel random graphs gives
$
\|L_{\mathrm{CS},n}(G_0)-L_{\mathrm{CS}}(G_0)\|_{\mathrm{op}}
\le
\frac{g_{\mathrm{eff}}}{2}
$
with probability at least \(1-\delta_0\), provided
$
n
\gtrsim
\frac{\log(1/\delta_0)}{r_x^{d_X}g_{\mathrm{eff}}^2}.
$
The factor \(r_x^{-d_X}\) is the effective inverse local sample size, and
\(g_{\mathrm{eff}}^{-2}\) is the usual spectral-resolution cost.
Choose
$
\delta_0
=
\frac{\delta'}{|\mathcal N|}.
$
A union bound over all \(G_0\in\mathcal N\) gives, with probability at least
\(1-\delta'\),
$
\sup_{G_0\in\mathcal N}
\|L_{\mathrm{CS},n}(G_0)-L_{\mathrm{CS}}(G_0)\|_{\mathrm{op}}
\le
\frac{g_{\mathrm{eff}}}{2}.
$
By Weyl's inequality \citep{weyl1949inequalities}, no eigenvalue can cross the spectral threshold
\(\tau_{\mathrm{spec}}\) as long as the perturbation is smaller than the
spectral margin \(g_{\mathrm{eff}}/2\). Therefore, the empirical eigenvalue count
\(\widehat w_n(G_0)\) equals the population width count \(w_{G_0}(P)\) for all
net elements. The stability margin then extends the equality from the net to
all \(G\in\mathcal G_K\).
Collecting the entropy, occupancy, and spectral-resolution terms, it suffices
that
$
n
\ge
C
\frac{
H_{\mathcal G}(\eta_{\mathrm{stab}}/2)
+
d_X\log(1/r_x)
+
\log(1/\delta')
}{
r_x^{d_X}g_{\mathrm{eff}}^2
}.
$
Under this condition,
$
\Pr\!\left(
\sup_{G\in\mathcal G_K}
|\widehat w_n(G)-w_G(P)|\ge 1
\right)
\le
\delta'.
$
\end{proof}

\subsection{Proof of Theorem~\ref{thm:safe-merge}}

\begin{proof}
Let \(U_a,U_b\) be candidate cells. The merge test has three components.
First, geometric admissibility
$
\widehat w_\varepsilon(U_a\cup U_b)=1
$
certifies that the union is geometrically compatible with a single structural
cell at scale \(\varepsilon\).
Second, contractive compatibility gives an empirical cross-cell Lipschitz
estimate
$
\widehat\kappa_{ab}\le \bar\gamma<1.
$
Assume the empirical estimate uniformly controls the population Lipschitz
modulus up to the sampling slack absorbed into \(\bar\gamma\). Then there exists
a map defined on \(U_a\cup U_b\) whose Lipschitz modulus is at most
$
\gamma'=\widehat\kappa_{ab}<1 .
$
Equivalently, one may define compatible predictors on \(U_a\) and \(U_b\) and
use Kirszbraun's extension theorem \citep{kirszbraun1934zusammenziehende} to extend the resulting Lipschitz map to the
union without increasing its Lipschitz constant.
Third, the empirical merge gap is below the merge threshold. By the definition
of \(\Gamma_{\min}\), compatible and incompatible pairs are separated at the
population level by at least \(\Gamma_{\min}\). Since the loss is bounded in
\([0,1]\), Hoeffding's inequality \citep{hoeffding1963probability} gives, for \(m\) samples per candidate pair,
$
\Pr\left(
|\widehat\Delta_{\mathrm{merge}}
-
\Delta_{\mathrm{merge}}|
>
\Gamma_{\min}/2
\right)
\le
2\exp(-c m\Gamma_{\min}^2)
$
for a universal constant \(c>0\). Taking a union bound over at most \(K\)
candidate merges, the merge decisions are all correct with probability at least
\(1-\delta'\) whenever
$
m
\ge
C\frac{\log(K/\delta')}{\Gamma_{\min}^2}.
$
On this event, the pair \(U_a,U_b\) is truly mergeable. Their union
admits a \(\gamma'\)-contractive predictor with risk bounded by the prescribed
\(\delta'\)-level. So \(U_a\cup U_b\) is \((\gamma',\delta')\)-contractive
with
$
\gamma'=\widehat\kappa_{ab}.
$
\end{proof}

\subsection{Proof of Theorem~\ref{thm:global-split-merge}}

\begin{proof}
Define
$
V(\Pi)=aI_{\mathrm{tot}}(\Pi)+bK(\Pi),
$
where
$
I_{\mathrm{tot}}(\Pi)=\sum_{k=1}^K I(U_k),
\qquad
K(\Pi)=K.
$
We first show that \(V\) decreases under splits (claim 1). If a cell \(U\) is split,
then by Assumption~\ref{ass:split-merge-regularity},
$
I(U^a)+I(U^b)\le (1-\alpha_{\mathrm{split}})I(U).
$
For a split to be triggered, \(I(U)\ge\varepsilon_0\). Hence the impurity
decrease is at least
$
I(U)-I(U^a)-I(U^b)
\ge
\alpha_{\mathrm{split}}I(U)
\ge
\alpha_{\mathrm{split}}\varepsilon_0.
$
The split increases the number of cells by one. Therefore
$
\Delta V_{\mathrm{split}}
\le
-a\alpha_{\mathrm{split}}\varepsilon_0+b.
$
By the choice
$
b<a\alpha_{\mathrm{split}}\varepsilon_0,
$
we have
$
\Delta V_{\mathrm{split}}<0.
$
Next consider a merge. A correct merge decreases the number of cells by one.
Let \(\varepsilon_{\mathrm{merge}}\) denote the maximum impurity increase allowed
by the safe merge test. Then
$
\Delta I_{\mathrm{tot}}\le \varepsilon_{\mathrm{merge}},
\qquad
\Delta K=-1.
$
Thus
$
\Delta V_{\mathrm{merge}}
\le
a\varepsilon_{\mathrm{merge}}-b.
$
By the choice
$
a\varepsilon_{\mathrm{merge}}<b,
$
we have
$
\Delta V_{\mathrm{merge}}<0.
$
Thus \(V\) strictly decreases at every valid split or merge step.
Let
$
\Delta_{\min}
:=
\min\{
a\alpha_{\mathrm{split}}\varepsilon_0-b,\,
b-a\varepsilon_{\mathrm{merge}}
\}>0.
$
Then each transform step decreases \(V\) by at least \(\Delta_{\min}\).
Moreover,
$
0\le V(\Pi)\le aP(X)+bK_{\max}.
$
Therefore the number of transform steps is at most
$
\frac{aP(X)+bK_{\max}}{\Delta_{\min}}.
$
This proves finite termination (claim 2).
It remains to characterize the terminal partition (claim 3). At termination, no split
fires. Hence every cell has impurity less than \(\varepsilon_0\); that is, the
partition is \(\varepsilon_0\)-pure.
Also at termination, no merge fires. By the robust merge condition, no two
cells sharing the same dominant true basin remain mergeable. Hence no two
terminal cells share a dominant basin.
Finally, basin visibility implies that every true basin \(B_j\) has enough mass
to be the dominant basin of at least one terminal cell. If a basin had
no cell in which it was dominant, its total mass would be spread across at
most \(K_{\max}\) cells as non-dominant impurity, contributing at most
\(K_{\max}\varepsilon_0\), contradicting
$
P(B_j)\ge K_{\max}\varepsilon_0+\varepsilon_0 .
$
Therefore each of the \(w\) true basins appears as a dominant basin of some
terminal cell, so \(K\ge w\). Since no two cells share a dominant basin and
there are only \(w\) basins, \(K\le w\). Hence \(K=w\).
The terminal partition has \(K=w\), is \(\varepsilon_0\)-pure, and has no
two cells sharing a dominant basin.
\end{proof}

\subsection{Proof of Theorem~\ref{thm:structural-erm-consistency}}

\begin{proof}
For each \(K\), let
$
R_K^\star
:=
\inf_{\Pi,G:|\Pi|=K} R(\Pi,G)
$
be the population structural risk at width \(K\), and let
$
\widehat R_{n,K}
:=
\inf_{\Pi,G:|\Pi|=K} \widehat R_n(\Pi,G)
$
be its empirical counterpart. The structural ERM selects
$
\widehat K_n
\in
\arg\min_{K\le K_{\max}}
\left\{
\widehat R_{n,K}+\mathrm{pen}_n(K)
\right\}.
$
Assume the following standard conditions.
First, there is a structural gap below width:
$
R_K^\star\ge R_w^\star+\eta_K
\quad
\forall K<w,
$
with \(\eta_K>0\).
Second, there is no population risk gain above the true width:
$
R_K^\star=R_w^\star
\quad
\text{for all }K\ge w.
$
Third, uniform convergence holds for each \(K\le K_{\max}\):
$
\sup_{\Pi,G:|\Pi|=K}
|\widehat R_n(\Pi,G)-R(\Pi,G)|
\le r_n(K),
\quad
r_n(K)\to0
$
almost surely.
Fourth, the penalty satisfies
$
\mathrm{pen}_n(K)\to0,
$
is increasing in \(K\), and for every \(K>w\),
$
\frac{
\mathrm{pen}_n(K)-\mathrm{pen}_n(w)
}{
r_n(K)+r_n(w)
}
\to\infty .
$
We show that \(\widehat K_n=w\) eventually almost surely.
For \(K<w\), uniform convergence gives
$
\widehat R_{n,K}
\ge
R_K^\star-r_n(K)
\ge
R_w^\star+\eta_K-r_n(K),
$
whereas
$
\widehat R_{n,w}
\le
R_w^\star+r_n(w).
$
Therefore
$
\widehat R_{n,K}+\mathrm{pen}_n(K)
-
\bigl(\widehat R_{n,w}+\mathrm{pen}_n(w)\bigr)
\ge
\eta_K-r_n(K)-r_n(w)
+
\mathrm{pen}_n(K)-\mathrm{pen}_n(w).
$
Since \(r_n(K),r_n(w)\to0\) and \(\eta_K>0\), the right-hand side is positive
for all sufficiently large \(n\). Thus underfitted models \(K<w\) are
eventually eliminated.
For \(K>w\), the no-gain condition gives \(R_K^\star=R_w^\star\). Uniform
convergence yields
$
\widehat R_{n,K}
\ge
R_w^\star-r_n(K),
\qquad
\widehat R_{n,w}
\le
R_w^\star+r_n(w).
$
Hence
$
\widehat R_{n,K}+\mathrm{pen}_n(K)
-
\bigl(\widehat R_{n,w}+\mathrm{pen}_n(w)\bigr)
\ge
\mathrm{pen}_n(K)-\mathrm{pen}_n(w)
-
r_n(K)-r_n(w).
$
By the penalty domination condition, this is positive for all sufficiently
large \(n\). Thus overfitted models \(K>w\) are eventually eliminated.
Consequently,
$
\widehat K_n=w
$
for all sufficiently large \(n\), almost surely.
Once \(\widehat K_n=w\), structural ERM is ordinary ERM over the fixed
width-\(w\) class. Uniform convergence then implies
$
R(\widehat\Pi_n,\widehat G_n)\to R_w^\star .
$
Finally, assume partition identifiability: for every \(\epsilon>0\), there
exists \(\rho(\epsilon)>0\) such that every width-\(w\) partition at distance
at least \(\epsilon\) from the population optimizer set
\(\mathcal P_w^\star\) has risk at least \(R_w^\star+\rho(\epsilon)\). If
$
\mathrm{dist}(\widehat\Pi_n,\mathcal P_w^\star)
\not\to0,
$
then for some \(\epsilon>0\) a subsequence remains at distance at least
\(\epsilon\), and hence has risk at least \(R_w^\star+\rho(\epsilon)\), which
contradicts
$
R(\widehat\Pi_n,\widehat G_n)\to R_w^\star .
$
Therefore
$
\mathrm{dist}(\widehat\Pi_n,\mathcal P_w^\star)\to0.
$
\end{proof}

\section{Proofs for Section~\ref{sec:metric-slingshot}}
\label{app:metric-slingshot-proofs}

This appendix gives proofs for the propositions, corollaries, and theorems in
Section~\ref{sec:metric-slingshot}.  We make explicit several standard
regularity assumptions that are implicit in the main text.

\subsection{Proof of Theorem~\ref{thm:contraction-transfer}}

\begin{proof}
Let \(x,x'\in U\). Since \(\phi(U)\subseteq Z_c\), both \(\phi(x)\) and
\(\phi(x')\) lie in the patch on which \(G_c^0\) is \(\gamma_Z\)-contractive.
Therefore
$
d_{H_c}\bigl(G_c^0(\phi(x)),G_c^0(\phi(x'))\bigr)
\le
\gamma_Z\, d_Z(\phi(x),\phi(x')).
$
Since \(\pi_c:H_c\to Y\) is \(L_{\pi,c}\)-Lipschitz,
$
d_Y(F_c(x),F_c(x'))
=
d_Y\!\left(
\pi_c(G_c^0(\phi(x))),
\pi_c(G_c^0(\phi(x')))
\right)
$
satisfies
$
d_Y(F_c(x),F_c(x'))
\le
L_{\pi,c}\,
d_{H_c}\bigl(G_c^0(\phi(x)),G_c^0(\phi(x'))\bigr).
$
Combining the two inequalities gives
$
d_Y(F_c(x),F_c(x'))
\le
L_{\pi,c}\gamma_Z d_Z(\phi(x),\phi(x')).
$
By the definition of the upper local distortion \(L_\phi(U)\),
$
d_Z(\phi(x),\phi(x'))
\le
L_\phi(U)d_X(x,x').
$
Thus
$
d_Y(F_c(x),F_c(x'))
\le
L_{\pi,c}\gamma_Z L_\phi(U)\, d_X(x,x').
$
Hence \(F_c\) is \(\gamma_X(U)\)-contractive with
$
\gamma_X(U)\le L_{\pi,c}\gamma_ZL_\phi(U).
$
If
$
L_{\pi,c}\gamma_ZL_\phi(U)<1,
$
then the contraction constant is strictly less than one, so \(F_c\) is a strict
contraction on \(U\).
\end{proof}

\subsection{Proof of Corollary~\ref{cor:funnel-feasibility}}

\begin{proof}
By Theorem~\ref{thm:contraction-transfer}, if \(U\) is routed to patch \(c\) and
$
L_{\pi,c}\gamma_ZL_\phi(U)\le \gamma,
$
then
$
d_Y(F_c(x),F_c(x'))
\le
\gamma d_X(x,x')
\qquad
\text{for all }x,x'\in U.
$
Thus \(F_c=\pi_c\circ G_c^0\circ\phi\) satisfies the contraction part of
\((\gamma,\delta)\)-feasibility.  The second assumed condition,
$
R_U(F_c)\le\delta,
$
is exactly the local-risk part of \((\gamma,\delta)\)-feasibility. Therefore
\(U\) is \((\gamma,\delta)\)-feasible through the slingshot.
\end{proof}

\subsection{Proof of Theorem~\ref{thm:risk-transfer-slingshot}}

\begin{proof}
Let \(F_c=\pi_c\circ G_c^0\circ\phi\). By assumption,
$
\sup_{x\in U}d_Y(F_c(x),f^*(x))\le \varepsilon_U.
$
Since \(\ell(\hat y,y)\) is \(L_\ell\)-Lipschitz in its first argument, for every
\((x,y)\) with \(x\in U\),
$
\ell(F_c(x),y)
\le
\ell(f^*(x),y)
+
L_\ell d_Y(F_c(x),f^*(x)).
$
Using the uniform approximation bound gives
$
\ell(F_c(x),y)
\le
\ell(f^*(x),y)+L_\ell\varepsilon_U.
$
Taking conditional expectation over \(P_U\) yields
$
R_U(F_c)
=
\mathbb E[\ell(F_c(X),Y)\mid X\in U]
\le
\mathbb E[\ell(f^*(X),Y)\mid X\in U]+L_\ell\varepsilon_U.
$
Hence
$
R_U(F_c)\le R_U(f^*)+L_\ell\varepsilon_U.
$
Now suppose
$
\mathcal E_{\mathrm{lat}}(U,c)\le \varepsilon_U.
$
By definition,
$
\mathcal E_{\mathrm{lat}}(U,c)
=
\inf_{\pi_c}
\sup_{x\in U}
d_Y\bigl(\pi_c(G_c^0(\phi(x))),f^*(x)\bigr).
$
Therefore, for every \(\eta>0\), there exists a readout \(\pi_c^\eta\) such that
$
\sup_{x\in U}
d_Y\bigl(\pi_c^\eta(G_c^0(\phi(x))),f^*(x)\bigr)
\le
\varepsilon_U+\eta.
$
Applying the first part of the proof gives
$
R_U(\pi_c^\eta\circ G_c^0\circ\phi)
\le
R_U(f^*)+L_\ell(\varepsilon_U+\eta).
$
Letting \(\eta\downarrow0\), or assuming the infimum is attained, gives the
claimed bound
$
R_U(\pi_c\circ G_c^0\circ\phi)
\le
R_U(f^*)+L_\ell\varepsilon_U.
$
\end{proof}

\subsection{Proof of Theorem~\ref{thm:width-transfer}}

\begin{proof}
Let
$
\{U_1,\dots,U_w\}
$
be a width-realizing cover of \(P\) in \(X\), so that
$
w=w(P;\gamma,\delta).
$
By Assumption~\ref{ass:latent-covering-regularity}, for each \(U_k\), the
embedded set \(\phi(U_k)\subseteq Z\) can be covered by at most
$
N_Z(U_k)
\le
\left\lceil
\frac{L_\phi(U_k)L_X(U_k)}{D_0^Z}
\right\rceil
$
latent patches, and each such latent patch is funnel-feasible.
For each \(k\), choose such a latent patch cover
$
\phi(U_k)\subseteq \bigcup_{j=1}^{N_Z(U_k)} Z_{k,j}.
$
Pulling this cover back through \(\phi\) gives a cover of \(U_k\):
$
U_k
\subseteq
\bigcup_{j=1}^{N_Z(U_k)}
\phi^{-1}(Z_{k,j})\cap U_k.
$
Since each \(Z_{k,j}\) is funnel-feasible by Assumption~\ref{ass:latent-covering-regularity},
the corresponding pulled-back region is feasible through the slingshot by
Corollary~\ref{cor:funnel-feasibility}. Therefore the full domain \(X\) admits a
latent contractive cover of size at most
$
\sum_{k=1}^w N_Z(U_k)
\le
\sum_{k=1}^w
\left\lceil
\frac{L_\phi(U_k)L_X(U_k)}{D_0^Z}
\right\rceil.
$
Thus
$
w_Z(P_\phi)
\le
\sum_{k=1}^w
\left\lceil
\frac{L_\phi(U_k)L_X(U_k)}{D_0^Z}
\right\rceil.
$
If, in addition,
$
L_\phi(U_k)\le \lambda,
\qquad
L_X(U_k)\le L_X
\quad
\text{for all }k,
$
then each summand satisfies
$
\left\lceil
\frac{L_\phi(U_k)L_X(U_k)}{D_0^Z}
\right\rceil
\le
\left\lceil
\lambda\frac{L_X}{D_0^Z}
\right\rceil.
$
Hence
$
w_Z(P_\phi)
\le
w(P;\gamma,\delta)
\left\lceil
\lambda\frac{L_X}{D_0^Z}
\right\rceil.
$
\end{proof}

\subsection{Proof of Theorem~\ref{thm:per-funnel-generalization}}

\begin{proof}
Let
$
\mathcal A_c:=\ell\circ\mathcal F_c
=
\{(x,y)\mapsto \ell(F(x),y):F\in\mathcal F_c\}.
$
By assumption, every function in \(\mathcal A_c\) is bounded in \([0,1]\).
Let \(S_c=\{(x_i,y_i)\}_{i=1}^{n_c}\) be the samples routed correctly to cell
\(U_c\). The standard Rademacher generalization inequality for bounded loss
classes states that, with probability at least \(1-\delta_c\),
$
\sup_{a\in\mathcal A_c}
\left|
\mathbb E[a]-\frac1{n_c}\sum_{i=1}^{n_c}a(x_i,y_i)
\right|
\le
2\mathfrak R_{n_c}(\mathcal A_c)
+
\sqrt{\frac{\log(2/\delta_c)}{2n_c}}.
$
Substituting \(a=\ell\circ F\) gives
$
\sup_{F\in\mathcal F_c}
|R_c(F)-\widehat R_c(F)|
\le
2\mathfrak R_{n_c}(\ell\circ\mathcal F_c)
+
\sqrt{\frac{\log(2/\delta_c)}{2n_c}}.
$
If the readout class \(\Pi_c\) has pseudo-dimension \(p_c\), and the maps
\(G_c^0\) and \(\phi\) are fixed, then the composed class
$
\mathcal F_c
=
\{x\mapsto \pi(G_c^0(\phi(x))):\pi\in\Pi_c\}
$
has pseudo-dimension at most \(p_c\). Since \(\ell\) is bounded and the
composition with a fixed loss preserves the standard pseudo-dimension
growth-function bound up to constants, the Rademacher complexity satisfies
$
\mathfrak R_{n_c}(\ell\circ\mathcal F_c)
=
O\!\left(
\sqrt{\frac{p_c\log n_c}{n_c}}
\right).
$
\end{proof}

\subsection{Proof of Corollary~\ref{cor:oracle-rate-slingshot}}

\begin{proof}
Assume the E-D-T cycle has converged to the correct width \(w\) and that
samples are routed correctly to the cells \(U_1,\dots,U_w\). Apply
Theorem~\ref{thm:per-funnel-generalization} to each cell \(U_c\) with
confidence parameter
$
\delta_c=\frac{\delta}{w}.
$
A union bound over \(c=1,\dots,w\) implies that, with probability at least
\(1-\delta\), all per-cell generalization bounds hold simultaneously.
For each cell, decompose the local excess risk into approximation and
estimation terms:
$
R_c(\widehat F_c)-R_c(F_c^*)
\le
\mathrm{Approx}_c
+
O\!\left(
\sqrt{\frac{p_c\log n_c+\log(w/\delta)}{n_c}}
\right),
$
where \(\mathrm{Approx}_c\) is the latent realization error or approximation
error on \(U_c\). Weighting by the cell probabilities \(P(U_c)\) and summing
over \(c\) gives
$
R(\widehat F)-R(F^*)
\le
\sum_{c=1}^w P(U_c)
\left[
\mathrm{Approx}_c
+
O\!\left(
\sqrt{
\frac{p_c\log n_c+\log(w/\delta)}{n_c}
}
\right)
\right].
$
Therefore, after trap resolution, the remaining funnel obeys the standard local
statistical rate.
\end{proof}

\subsection{Proof of Proposition~\ref{prop:architectural-decoupling}}

\begin{proof}
Let \(\theta_\Sigma\) denote the parameters of the indexer, and let
\(\theta_G\) denote the parameters of the metric learners, including the
pre-built contraction maps and local readouts where appropriate.
By assumption, \(\Sigma\) is trained only by mismatch or novelty signals, not
by the task-loss gradient. Therefore the computational graph for the funnel
loss \(L_{\mathrm{funnel}}\) contains no differentiable path into
\(\theta_\Sigma\). Hence
$
\nabla_{\theta_\Sigma}L_{\mathrm{funnel}}=0.
$
Similarly, the trap loss \(L_{\mathrm{trap}}\) is used only to train the
indexer or novelty mechanism. The pre-built contraction maps \(G_c^0\) are
frozen, and the readout \(\pi_c\) for context \(c\) is updated only on samples
already routed to \(c\). Therefore, the trap-loss computational graph contains
no update path into the metric-learner parameters \(\theta_G\). Hence
$
\nabla_{\theta_G}L_{\mathrm{trap}}=0.
$
Thus, the trap and funnel gradients decouple.
\end{proof}

\subsection{Proof of Proposition~\ref{prop:no-circular-dependency}}

\begin{proof}
We verify that none of the four stages of the bidirectional bootstrap requires
knowledge of the unknown width \(w(P)\).
Stage 1) the where-to-what stage, learns \(\phi:X\to Z\) by self-supervised
prediction of latent or spatial coordinates. This objective is defined from
observed coordinates or predictive signals and does not require knowing the
number of structural basins.
Stage 2) the what-to-where stage, uses the learned map \(\phi\) to embed task
inputs as \(z=\phi(x)\). This is a forward computation and likewise does not
require knowing \(w(P)\).
Stage 3) trap discovery, trains \(\Sigma\) and estimates structural width using
the CS operator in \(Z\). The CS operator uses observed latent distances,
predictions, and mismatch signals. It estimates the number of contractive
components; it does not take the true width as input.
Stage 4) funnel reuse, solves local tasks using the active latent patch and its
readout. This stage is conditioned on the currently discovered routing and does
not assume the final number of contexts.
Therefore, each stage is well-defined without prior knowledge of \(w(P)\). The
width is an output of the discovery process rather than an input to it.
\end{proof}

\subsection{Proof of Theorem~\ref{thm:local_bootstrap_convergence}}

The theorem in the main text refers to two local stability inequalities. For
clarity, we state them explicitly here. Let
$
a_t:=d_\Phi(\phi_t,[\phi^\star]),
\quad
b_t:=d_{\mathrm{part}}(\Sigma_t,\Sigma^\star),
$
where \(d_\Phi\) is a distance on embeddings modulo context-preserving
transformations and \(d_{\mathrm{part}}\) is a partition distance modulo
relabeling.
Assume the CS-indexing update \(\mathcal S\) and the coordinate update
\(\mathcal T\) satisfy
$
d_{\mathrm{part}}(\mathcal S(\phi),\Sigma^\star)
\le
C_\Sigma d_\Phi(\phi,[\phi^\star])
$
and
$
d_\Phi(\mathcal T(\phi,\Sigma),[\phi^\star])
\le
q_\Phi d_\Phi(\phi,[\phi^\star])
+
C_\Phi d_{\mathrm{part}}(\Sigma,\Sigma^\star).
$
The bootstrap iterates are
$
\Sigma_t=\mathcal S(\phi_t),
\quad
\phi_{t+1}=\mathcal T(\phi_t,\Sigma_t).
$

\begin{proof}
Define the bootstrap error
$
\mathcal E_t
:=
\max\left\{
a_t,\frac{b_t}{C_\Sigma}
\right\},
$
with the convention \(\mathcal E_t=a_t\) if \(C_\Sigma=0\).
Assume first \(C_\Sigma>0\). Since \(\Sigma_t=\mathcal S(\phi_t)\), the
CS-indexing stability inequality gives
$
b_t
=
d_{\mathrm{part}}(\Sigma_t,\Sigma^\star)
=
d_{\mathrm{part}}(\mathcal S(\phi_t),\Sigma^\star)
\le
C_\Sigma a_t.
$
Hence
$
\frac{b_t}{C_\Sigma}\le a_t,
$
so
$
\mathcal E_t=a_t.
$
Now use the coordinate-update stability inequality:
$
a_{t+1}
=
d_\Phi(\phi_{t+1},[\phi^\star])
=
d_\Phi(\mathcal T(\phi_t,\Sigma_t),[\phi^\star])
\le
q_\Phi a_t+C_\Phi b_t.
$
Using \(b_t\le C_\Sigma a_t\), we obtain
$
a_{t+1}
\le
(q_\Phi+C_\Phi C_\Sigma)a_t.
$
Let
$
q:=q_\Phi+C_\Phi C_\Sigma.
$
By assumption \(q<1\), and thus
$
a_{t+1}\le q a_t.
$
Since \(\mathcal E_t=a_t\) after each indexing step, this implies
$
\mathcal E_{t+1}\le q\mathcal E_t.
$
Iterating gives
$
\mathcal E_t\le q^t\mathcal E_0.
$
Therefore \(\mathcal E_t\to0\) linearly, which means
$
d_\Phi(\phi_t,[\phi^\star])\to0
$
and, since \(b_t\le C_\Sigma a_t\),
$
d_{\mathrm{part}}(\Sigma_t,\Sigma^\star)\to0.
$
Thus the bootstrap converges linearly to
$
([\phi^\star],\Sigma^\star).
$
If \(C_\Sigma=0\), then the indexing stability inequality gives \(b_t=0\) in
the local neighborhood. The coordinate update becomes
$
a_{t+1}\le q_\Phi a_t.
$
Since \(q=q_\Phi<1\), the same conclusion follows.
\end{proof}

\subsection{Proof of Theorem~\ref{thm:grid-spacing}}

\begin{proof}
Let the dynamic range be
$
R=\frac{L_{\max}}{L_{\min}}.
$
Taking logarithms transforms the interval of path lengths
$
[L_{\min},L_{\max}]
$
into the interval
$
[\log L_{\min},\log L_{\max}]
$
of length
$
\log R.
$
A multiplicative contraction scale \(s_m\) corresponds to a point
$
u_m:=\log s_m
$
on this logarithmic interval.
The minimax log-covering problem asks us to choose \(M\) points
$
u_1<\cdots<u_M
$
so that the largest uncovered logarithmic gap is minimized. On an interval,
the minimax solution is to space the points uniformly. Therefore
$
u_m
=
u_1+(m-1)\Delta,
$
where
$
\Delta=\frac{\log L_{\max}-\log L_{\min}}{M-1}
=
\frac{\log R}{M-1}.
$
Exponentiating gives
$
s_m
=
s_1 e^{(m-1)\Delta}
=
s_1\left(e^\Delta\right)^{m-1}.
$
Thus the optimal ratio is
$
r^*
=
e^\Delta
=
\exp\!\left(\frac{\log R}{M-1}\right)
=
R^{1/(M-1)}.
$
Hence the minimax log-covering solution is geometric:
$
s_m=s_1(r^*)^{m-1}.
$
\end{proof}

\subsection{Proof of Corollary~\ref{cor:slingshot-width-estimation}}

\begin{proof}
The proof of the uniform width-estimation theorem from
Section~\ref{sec:width-estimation} depends on three terms: predictor entropy,
geometric occupancy, and spectral resolution. The predictor entropy term
$
H_{\mathcal G}(\eta_{\mathrm{stab}}/2)
$
is unchanged by replacing \(X\) with \(Z\), because the predictor class being
searched is the same up to composition with the fixed embedding. The spectral
resolution term
$
g_{\mathrm{eff}}^{-2}
$
is also unchanged in form, although the numerical value of \(g_{\mathrm{eff}}\)
is now computed for the CS graph in \(Z\).
The only structural change is the occupancy term. In the original space \(X\),
local neighborhoods of radius \(r_x\) in a \(d_X\)-dimensional metric space
require sample complexity proportional to
$
r_x^{-d_X}.
$
If the CS operator is instead built in \(Z\) using metric \(d_Z\) and local
scale \(r_z\), then the corresponding metric occupancy cost is
$
r_z^{-d_Z}.
$
Replacing \(d_X,r_x\) by \(d_Z,r_z\) in the uniform width-estimation theorem
therefore gives
$
n
\ge
C\cdot
\frac{
H_{\mathcal G}(\eta_{\mathrm{stab}}/2)
+
d_Z\log(1/r_z)
+
\log(1/\delta')
}{
r_z^{d_Z}g_{\mathrm{eff}}^2
}.
$
When \(d_Z=2\), the geometric occupancy factor is \(r_z^{-2}\), rather than
\(r_x^{-d_X}\).
\end{proof}

\subsection{Proof of Theorem~\ref{thm:metric-slingshot-funnel}}

\begin{proof}
Fix a cell \(U_c\). Assumptions (1)--(3) imply that
$
\phi(U_c)\subseteq Z_c,
$
that \(G_c^0\) is \(\gamma_Z\)-contractive on \(Z_c\), and that \(\pi_c\) is
\(L_{\pi,c}\)-Lipschitz with
$
L_{\pi,c}\gamma_ZL_\phi(U_c)<1.
$
By Theorem~\ref{thm:contraction-transfer}, the composed predictor
$
F_c=\pi_c\circ G_c^0\circ\phi
$
is a strict contraction on \(U_c\).
Assumption (4) says that the latent realization error on \(U_c\) is at most
\(\varepsilon_c\). By Theorem~\ref{thm:risk-transfer-slingshot}, there exists
a readout whose local approximation risk satisfies
$
R_c(F_c)
\le
R_c(f^*)+L_\ell\varepsilon_c
$
up to statistical estimation error.
Now apply Theorem~\ref{thm:per-funnel-generalization} to each cell \(U_c\)
with confidence parameter \(\delta/w\). A union bound over the \(w\) cells
implies that, with probability at least \(1-\delta\), all per-cell
generalization bounds hold simultaneously. For cell \(U_c\), this gives the
estimation term
$
O\!\left(
\sqrt{
\frac{p_c\log n_c+\log(w/\delta)}{n_c}
}
\right).
$
Combining approximation and estimation terms yields
$
R_c(F_c)
\le
R_c(f^*)+
L_\ell\varepsilon_c+
O\!\left(
\sqrt{
\frac{p_c\log n_c+\log(w/\delta)}{n_c}
}
\right)
$
simultaneously for all \(c=1,\dots,w\).
Finally, after a cell has stabilized and the trap no longer needs to perform
additional structural search inside it, inference on that cell is the direct
routed computation
$
x\mapsto \pi_c(G_c^0(\phi(x))).
$
This is a single forward pass through the slingshot pipeline and therefore
requires only the amortized cost of evaluating the fixed maps and the active
readout.
\end{proof}
\end{document}